%% file: main.tex
\documentclass[11pt]{article}
\usepackage[margin =1in]{geometry}

\usepackage{pifont}
\usepackage{wrapfig}
\usepackage{comment}
\usepackage{subfig}
\usepackage[font=footnotesize]{caption}
\usepackage{amssymb,amsthm}
\usepackage{pifont}
\usepackage{amsmath}
\usepackage{enumerate}
\usepackage{hyperref}
\usepackage{color}
\usepackage{cleveref}
\usepackage{algorithmic}
\usepackage{tablefootnote}
\usepackage[utf8]{inputenc} 
\usepackage[T1]{fontenc}    
\usepackage{hyperref}       
\usepackage{url}            
\usepackage{booktabs}       
\usepackage{amsfonts}       
\usepackage{nicefrac}       
\usepackage{microtype}      
\usepackage{xcolor}         
\numberwithin{equation}{section}
\usepackage{float} 

\input{liweimacros}
\newcommand{\cmark}{\ding{51}}%
\newcommand{\xmark}{\ding{55}}
\newcommand{\size}{\rho}

\newcommand{\sizei}{c_\rho}

\title{Algorithmic Regularization in Model-free Overparametrized Asymmetric Matrix Factorization}

\author{Liwei Jiang\thanks{School of Operations Research and Information Engineering, Cornell University. Ithaca, NY 14850, USA; lj282@cornell.edu}
\qquad
Yudong Chen\thanks{Department of Computer Sciences,
University of Wisconsin-Madison. Madison, WI 53706, USA;
 yudong.chen@wisc.edu} 
\qquad 
Lijun Ding\thanks{Wisconsin Institute for Discovery, University of Wisconsin-Madison, Madison, WI 53706, USA;  lding47@wisc.edu}}

\begin{document}

\maketitle

\begin{abstract}%
  We study the asymmetric matrix factorization problem under a natural nonconvex formulation with arbitrary overparametrization. The \emph{model-free} setting is considered, with minimal assumption on the rank or singular values of the observed matrix, where the global optima provably overfit. We show that vanilla gradient descent with small random initialization \emph{sequentially} recovers  the principal components of the observed matrix. Consequently, when equipped with proper early stopping, gradient descent produces the best low-rank approximation of the observed matrix without explicit regularization. We provide a sharp characterization of the relationship between the approximation error, iteration complexity, initialization size and stepsize. Our complexity bound is almost dimension-free and depends logarithmically on the approximation error, with significantly more lenient requirements on the stepsize and initialization compared to prior work. Our theoretical results provide accurate prediction for the behavior gradient descent, showing good agreement with numerical experiments.%
\end{abstract}

\newpage 
\tableofcontents
\newpage

\section{Introduction} \label{sec: intro}

Let $X\in \RR^{m\times n}$ be an arbitrary given matrix. In this paper, we study the following nonconvex objective function for asymmetric matrix factorization
\begin{align}
    f(F,G) := \frac{1}{2}\fnorm{FG^\top - X}^2 \quad 
    \label{eq: objective}\tag{\text{$\mathcal{M}$}}
\end{align}
and the associated vanilla gradient descent dynamic: 
\begin{align}
    \text{initialize} \; (F_0,G_0); \;\;\text{run iteration}\;
    (F_{t+1},G_{t+1}) = (F_t,G_t) - \eta \nabla f(F_t,G_t), 
    \label{eq: GD}\tag{\text{$\mathcal{GD}$-$\mathcal{M}$}}
\end{align}
where $(F,G)\in \RR^{m\times k}\times \RR^{n\times k}$ are the factor variables, $k\ge 1$ is a \emph{user-specified} rank parameter, $\eta>0$ is the stepsize, and $\nabla f(F,G) = ((FG^\top - X)G, (FG^\top - X)^\top F).$

In many statistical and machine learning settings \cite{fan2020understanding}, the observed matrix $X$ takes the form
$ \label{eq: observedX}
    X = X_\natural + E, 
$
where $X_\natural$ is an unknown low-rank matrix to be estimated, and $E$ is the additive error/noise. Gradient descent applied to the objective~\eqref{eq: objective} is a natural approach for computing an estimate of $X_\natural$. Such an estimate can in turn be used as an approximate solution in more complicated nonconvex matrix estimation problems, such as matrix sensing \cite{recht2010guaranteed},  matrix completion \cite{candes2010power} and even nonlinear problems like the Single Index Model \cite{fan2020understanding}. In fact, the gradient descent procedure~\eqref{eq: GD} is often used, explicitly or implicitly, as a subroutine in more sophisticated algorithms for these problems. As such, characterizing the dynamics of~\eqref{eq: GD} provides deep intuition for more general problems and is considered an important first step for 
understanding various aspects of (linear) neural networks \cite{du2018algorithmic,ye2021global}.

Despite the apparent simplicity of the  dynamic~\eqref{eq: GD}, our understanding of its behavior remains limited, especially for general settings of $X$ and $k$. Results in existing work often only apply to symmetric, exactly low-rank matrices $X$ or specific choices of $k$. Many of them impose strong assumption on the initialization $(F_0, G_0)$, only provides asymptotic analysis, or has order-wise suboptimal error and iteration complexity bounds. We discuss these existing results and the associated challenges in greater details later.

Of particular interest is the \emph{overparametrization} regime, which is common in modern  machine learning paradigms~\cite{huang2019gpipe,kolesnikov2020big,tan2019efficientnet}. In the context of the objective~\eqref{eq: objective}, overparametrization means choosing the rank parameter $k$ to be larger than what is statistically necessary, e.g., $k=\min\{m,n\} \gg $ rank$(X_\natural)$. Doing so, however, necessarily leads to overfitting in general. Indeed, with $k=\min\{m,n\}$, any global optimum of~\eqref{eq: objective} is simply (a full factorization of) $X$ itself and overfits the noise/error in $X$, therefore failing to provide a useful estimate for $X_\natural$. 
Moreover, as can be seen from the numerical results in Figure \ref{fig: smGDeearlystop}, gradient descent~\eqref{eq: GD} with random initialization asymptotically converges to such a global minimum, with a vanishing training error $\fnorm{FG^\top - X}$ (dashed lines) but a large test error $\fnorm{FG^\top - X_\natural}$ (solid lines).

\begin{figure}[H]
  \begin{center}
    \includegraphics[height=0.3\textwidth, clip, trim = 0 0 0 0]{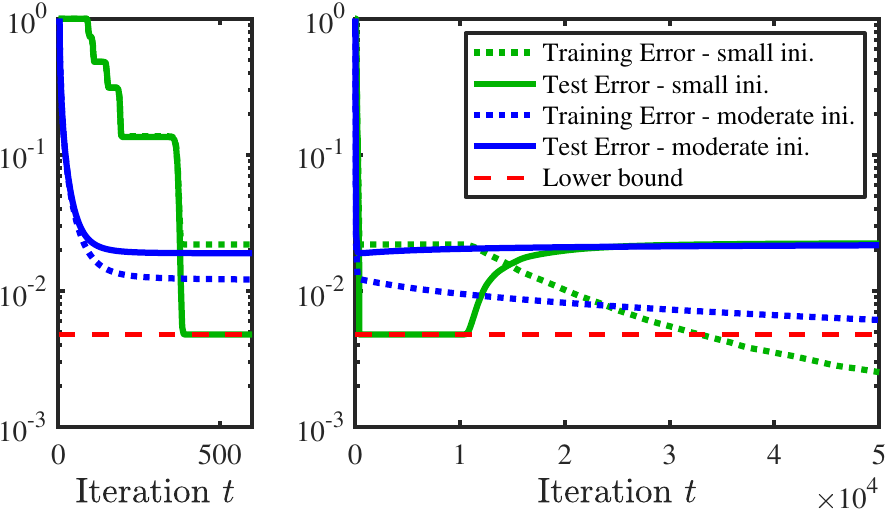}
    \vspace{-0.2cm}
  \end{center}
  \caption{
  Evolution of training error $\fnorm{F_tG_t^\top-X}$ and testing error $\fnorm{F_tG_t^\top -X_\natural}$ of \eqref{eq: GD} with $k=200$ under the model $X=X_\natural +E$, where  $X_\natural\in \mathbb{R}^{250\times 200}$ has rank $r=5$ and $\fnorm{X_\natural}=1$, and the noise matrix $E$ has i.i.d.\ zero-mean Gaussian entries and expected operator norm $\mathbb{E}\|E\|\approx 2\times 10^{-3}$. 
  The green lines use small initialization ($\fnorm{F_0},\fnorm{G_0}\ll \fnorm{X_\natural}$), and the blue lines use moderate initialization ($\fnorm{F_0},\fnorm{G_0}\asymp \fnorm{X_\natural}$). The red line represents the optimal statistical error $\sqrt{r} \|E\|$. The left panel zooms in the first 600 iterations.
  }\label{fig: smGDeearlystop}
\end{figure}

A careful inspection of Figure~\ref{fig: smGDeearlystop}, however, reveals an interesting phenomenon: gradient descent with \emph{small} random initialization achieves a small (and near optimal) test error before it eventually overfits; on the other hand, this behavior is not observed with moderate initialization.
We note that similar phenomena have been empirically observed in many other statistical and machine learning problems, where vanilla gradient descent coupled with \emph{small random initialization} (\texttt{SRI}) and \emph{early stopping} (\texttt{ES}) has good generalization performance, even with overparametrization \cite{woodworth2020kernel,ghorbani2020neural,prechelt1998early,wang2021early,li2018algorithmic,stoger2021small}. This observation motivates us to theoretically characterize, both qualitatively and quantitatively, the behavior of gradient descent~\eqref{eq: GD} and the \emph{algorithmic regularization} effect of \texttt{SRI} and \texttt{ES}.

\paragraph{Our results} 

We relate the dynamics of~\eqref{eq: GD} to the best low-rank approximations of $X$, defined as $X_r = \argmin_{\rank(Y)\leq r}\fnorm{Y-X}^2$ for $r=1,2,\ldots$ Our main results, Theorems~\ref{thm: mainthmallsigs} and~\ref{thm: mainthmonesig}, establish the following: 
\vspace{5pt}
\begin{quote}
    The iteration \eqref{eq: GD} with \texttt{SRI} \emph{sequentially} approaches the principal components of $X$, and proper \texttt{ES} produces the \emph{best low-rank approximation} of~$X$.
\end{quote}
\vspace{5pt}
Specifically, we show that for each $r\in[1,k]$, there exists a (range of) stopping time $t$ such that~\eqref{eq: GD} terminated at iteration $t$ produces an approximate $X_r$. Moreover, we provide a \emph{sharp} characterization of the scaling relationship between the approximation error, iteration complexity, initialization size and stepsize. This quantitative characterization agrees well with numerical experiments.

It is known that under many statistical models where $X$ is an observed noisy version of some structured matrix $X_\natural$, the matrix $X_r$ with an appropriate $r$ is a statistically optimal estimator of $X_\natural$ \cite{chatterjee2015matrix,fan2020understanding}. Our results thus imply that gradient descent with $\texttt{SRI}$ and $\texttt{ES}$ learns such an optimal estimate, even with overparametrization and no explicit regularization. In fact, our results are more general, applicable to any observed matrix $X$ and not tied to the existence of a ground truth $X_\natural$.

We emphasize that we do not claim that vanilla gradient descent is a more efficient way for computing $X_r$ for a given $r$ than standard numerical procedures (e.g., via singular value decomposition). Rather, our focus is to show, rigorously and quantitatively, that overparametrized gradient descent, a common and fundamental algorithmic paradigm, has an implicit regularization effect characterized by a deep connection to the computation of best low-rank approximation.

\paragraph{Analysis and challenges} 

Our analysis elucidates the mechanism that gives rise to the algorithmic regularization effect of small initialization and early stopping: starting from a small initial $F_0 G_0^\top $, the singular values of the iterate $F_t G_t^\top $ approach those of $X$ at geometrically different rates, hence $F_t G_t^\top $ approximates $X_1, X_2, \ldots$ sequentially (we elaborate in Section~\ref{section: earlyStoppingAndSmallIntialization}). While the intuition is simple, a rigorous and sharp analysis is highly non-trivial,  due to the following challenges:

    (i) \emph{Model-free setting.} Most existing work assumes that $X$ is (exactly or approximately) low-rank with a sufficiently large singular value gap $\delta$ between the $r$-th and $(r+1)$-th singular values of $X$ \cite{li2018algorithmic,zhuo2021computational,ye2021global,fan2020understanding,stoger2021small}. We allow for an arbitrary nonzero $\delta$ and characterize its impact on the stopping time and final error. In this setting, the ``signal'' $X_r$ may have magnitude arbitrarily close to that of the ``noise'' $X-X_r$ \emph{even in operator norm}.

    (ii) \emph{Asymmetry}.
    Since the objective~\eqref{eq: objective} is invariant under the rescaling $(F,G) \to (cF, c^{-1} G)$, the magnitudes of $F$ and $G$ may be highly imbalanced, in which case the gradient has a large Lipschitz constant. This issue is well recognized to be a primary difficulty in analyzing the gradient dynamics~\eqref{eq: GD} even without overparametrization \cite{du2018algorithmic}. Most previous works either restrict to the symmetric positive semidefinite formulation  \cite{li2018algorithmic,zhuo2021computational,stoger2021small,zhang2021preconditioned, ma2021sign,ding2021rank,
zhang2021sharp}, or add an explicit regularization term of the form $\fnorm{F^\top F-G^\top G}^2$ \cite{tu2016low,zheng2016convergence}.

    (iii) \emph{Trajectory analysis and cold start}. 
    As the desired $X_r$ is not a local minimizer of \eqref{eq: objective}, our analysis is inherently trajectory-based and initialization-specific. 
    
    Random initialization leads to a cold start: the initial iterates $F_t G_t^\top$ are far from and nearly orthogonal to $X_r$ when the dimensions $(m,n)$ are large.
    More precisely, assuming the SVD $X_r = U\Sigma V^\top$, one may measure the relative signal strength of $(F_t,G_t)$ by  
    \begin{equation}\label{eq: initialiSignal}
        \texttt{sig}(F_t,G_t) := \frac{\min\{\sigma_r(U^\top F_t),\sigma_r(V^\top G_t)\}}{\max\{\sigma_1((I-UU^\top)F_t),\sigma_1((I-VV^\top )G_t)\}},
    \end{equation}
    which is the ratio between the projections of $F_t$ and $G_t$ to the column/row space of $X$ and the projections to the complementary space. 
    Most existing work requires $\texttt{sig}(F_0,G_0)$ to be larger than a universal constant \cite{fan2020understanding,li2018algorithmic}, which does not hold when $k,r=o(m)$.\footnote{For $F_0,G_0$ with i.i.d.\ standard Gaussian entries, we have w.h.p.\ $\texttt{sig}(F_0,G_0) \lesssim \frac{\sqrt{k}+\sqrt{r-1}}{|\sqrt{m-r}-\sqrt{k}|}$ for $m\asymp n$.\label{ft: sig}}

    (iv) \emph{General rank overparametrization}. Our result holds for any choice of $k$ with $k\geq r$. The work in \cite{ye2021global,ma2021beyond} only considers the exact-parametrization setting $k=r$. The work \cite{fan2020understanding} assumes the specific choice $k=2m+2n$; as mentioned in the previous paragraph, the setting with a smaller $k$ involves additional challenges due to cold-start.

\section{Related work}
\label{sec: related}

The literature on gradient descent for matrix factorization is vast; see  \cite{chen2018harnessing,chi2019nonconvex} for a survey. Most prior work focuses on the exact parametrization setting $k=r$ (where $r$ is the target rank or the rank of a ground truth matrix $X_\natural$) and requires an explicit regularizer $\fnorm{F^\top F - G^\top G}^2$. More recent work in \cite{ding2021rank,du2018algorithmic, fan2020understanding,  li2018algorithmic,ma2021beyond,ma2021sign,stoger2021small,ye2021global, zhang2021preconditioned,zhang2021sharp,zhuo2021computational}, discussed in the last section, studies (overparametrized) matrix factorization and implicit regularization. Below we discuss recent results that are most related to ours.


The work~\cite{ye2021global} also provides recovery guarantees for vanilla gradient descent \eqref{eq: GD} with random small initialization. Their result only applies to the setting where the matrix $X$ has exactly rank $r$ and $k=r$, i.e., with exact  parametrization. Moreover, their choice of stepsize is quite conservative and consequently their iteration complexity scales proportionally with dimension $m+n$. In comparison, we allow for significantly larger stepsizes and establish almost dimension-free iteration complexity bounds.

The work~\cite{fan2020understanding} considers a wide range of statistical problems with a \emph{symmetric} ground truth matrix $X_\natural$ and shows that $X_\natural$ can be recovered with near optimal statistical errors using gradient descent for \eqref{eq: objective} with  $FG^\top$ replaced by  $FF^\top -GG^\top$.
While one may translate their results to the asymmetric setting via a dilation argument, doing so requires the specific rank parametrization $k =2m+2n$. This restriction allows for a decoupling of the dynamics of different singular values, which is essential to their analysis. While this decoupled setting provides intuition for the general setting (as we elaborate in Section~\ref{section: earlyStoppingAndSmallIntialization}), the same analysis no longer applies for other values of $k$, e.g., $k=2m+2n-1$, for which the decoupling ceases to hold. Moreover, a smaller value of $k$ leads to the cold start issue, as discussed in footnote~\ref{ft: sig}.

The work in~\cite{chou2020gradient} studies the deep matrix factorization problem. While on a high level their results deliver a message similar to our work, namely gradient descent sequentially approaches the principal components of $X$, the technical details differ significantly. In particular, they results only apply to symmetric $X$ and guarantee recovery of the positive semidefinite part of $X$. Their analysis relies crucially on the assumption $k=m=n$, a specific identity initialization scheme and the resulting decoupled dynamics, which do not hold in the general setting as discussed above. A major contribution of our work lies in handling the entanglement of singular values resulted from general overparametrization, asymmetry, and random initialization. 

In Section~\ref{section: auxiliary_lemmas}, we provide additional discussion on related work.

\section{Intuitions and the symmetric setting}
\label{section: earlyStoppingAndSmallIntialization}

In this section, we illustrate the behavior of gradient descent with small random initialization in the setting with a symmetric $X$, and explain the challenges for generalizing to asymmetric $X$. 

Consider a simple example where $m=n=k=2$,
and $X = \diag(\lambda_1,\lambda_2)$ is a positive semidefinite diagonal matrix that is approximately rank-1, i.e., $ \lambda_{2}\le \frac{1}{10}\lambda_1$. We consider the natural symmetric objective $f(F) = \frac{1}{4}\fnorm{FF^\top - X}^2$ and the associated gradient descent dynamic
$
    F_{t+1} = F_t - \eta(F_tF_t^\top -X)F_t
$,%
\footnote{This is equivalent to~\eqref{eq: GD} with initialization $F_0=G_0$.}
with initialization $F_0= \rho \sqrt{\lambda_1} I$ for some small $\rho>0$. 

It is easy to see that $ F_t$ is  \emph{diagonal} for all $t \ge 0$, and the $i$-th diagonal element of $F_t$, denoted by $f_{i,t}$, is updated as
\begin{equation}\label{eq: eigenvalueDynamic}
    f_{i,t+1} = f_{i,t}(1+\eta \lambda_i - \eta f_{i,t}^2),\quad  i =1,2.
\end{equation}
Thus, the dynamics of the two eigenvalues decouple and can be analyzed separately.
In particular, simple algebra shows that
(i) when $f_{1,t} < \sqrt{\frac{\lambda_1}{2}}$, 
$f_{1,t}$ increases geometrically by a factor of $1+\eta \lambda_1 - \eta f_{1,t}^2 \ge 1+ \frac{\eta \lambda_1}{2} $, i.e., $|f_{1,t+1}| \ge (1+ \frac{\eta \lambda_1}{2}) |f_{1,t}|$, and (ii) when $f_{1,t}\ge \sqrt{\frac{\lambda_1}{2}}$, $f_{1,t}$ converges to $\sqrt{\lambda_1}$ geometrically with a factor of $(1-\frac{\eta \lambda_1}{2})$, i.e., $|f_{1,t+1}-\sqrt{\lambda_1}|\le (1-\frac{\eta \lambda_1}{2})|f_{1,t}-\sqrt{\lambda_1}|.$
In a similar fashion, $f_{2,t}$ converges to the second eigenvalue $\sqrt{\lambda_2}$. It follows that the gradient descent iterate $F_t F_t^\top $ converges to the observed matrix $X$ as $t$ goes to infinity.

What makes a difference, however, is that $f_{2,t}$ converges at an exponentially slower rate than $f_{1,t}$. In particular, assuming the stepsize $\eta$ is sufficiently small, we can show that $f_{2,t}$ is always nonnegative and satisfies $f_{2,t+1}\le (1+\frac{\eta \lambda_1}{10})f_{2,t}$.
Note that the growth factor $1+\frac{\eta \lambda_1}{10}$ is smaller than $1+\frac{\eta \lambda_1}{2}$, the growth factor for $f_{1,t}$. We conclude that 
\begin{equation}\label{eq: mainIntuition}
    \textit{With small initialization, larger eigenvalues converge (exponentially) faster.}
\end{equation}

Thanks to the property~\eqref{eq: mainIntuition},  the gradient descent trajectory approaches the principal components of $X$ one by one. In particular, if the initial size $\rho$ is sufficiently small, then \eqref{eq: mainIntuition} implies the existence of time window for $t$ during which  $f_{1,t}$ is close to $\lambda_1$
while $f_{2,t}$ remains close to its initial value (i.e., close to $0$).
If we terminate at a time $t$ within this window, then the gradient descent output satisfies $ F_t F_t^\top \approx X_1$. If we continue the iteration, then $f_{2,t}$ eventually grows away from $0$ and  converges quickly to $\sqrt{\lambda_2}$, yielding $F_tF_t^\top \approx X_2$. 

In Section~\ref{section: auxiliary_lemmas}, we generalize the above argument to any $m\times m$ symmetric positive semi-definite $X$  via a diagonalization argument, which shows that $F_t F_t^\top$ approaches $X_1,X_2,X_3,\ldots$ sequentially.  
However, this simple derivation breaks down for general rectangular $F_t$ and $X$, since the dynamics of eigenvalues no longer decouple as in \eqref{eq: eigenvalueDynamic}; rather, they have complicated dependence on each other, as can be seen in the proof of our main theorems. One of the main contributions of our proof is to rigorously establish the property~\eqref{eq: mainIntuition} despite this complicated dependence.

\section{Main results and analysis}
\label{sec:main}

In this section, we present our main theorems on the trajectory of gradient descent under small random initialization and early stopping. We also outline the analysis, deferring the full proof to the appendices.

\subsection{Main theorems}

Recall that $X \in \mathbb{R}^{m\times n}$ is a general rectangular matrix and denote its singular values by $\sigma_1 \ge \cdots \ge \sigma_{\min\{m,n\}} \ge 0$. Fix $r \in [0,  \min\{k,m,n\}]$ and define the $r$-th condition number $\kappa_r := \frac{\sigma_1}{\sigma_r}$. Suppose that the gradient descent dynamic~\eqref{eq: GD} is initialized with $(F_0, G_0) = \frac{\rho}{3\sqrt{m+n+k}}(\tilde F_0, \tilde G_0)$, where $\rho>0$ is a size parameter and $\tilde F_0 \in \RR^{m\times k}, \tilde G_0 \in \RR^{n\times k}$ have i.i.d.\ $N(0,\sigma_1)$ entries. The operator norm is denoted by $\|\cdot\|$.

Below we state two theorems under one of the following assumptions: 
\begin{assumption}\label{assum: allrvalues}
The first $r+1$ singular values are distinct, i.e., $\sigma_1 > \cdots >\sigma_r >\sigma_{r+1}$.
\end{assumption}
\begin{assumption}\label{assum: rthvalue}
The $r$-th and $(r+1)$-th singular values are distinct, i.e., $\sigma_r > \sigma_{r+1}.$
\end{assumption}
Both theorems are high probability statements (w.r.t.\ random initialization); we refer to Theorem~\ref{thm: mainthmgeneralrestated} in the appendix for the precise value of the probability. 

Our first theorem shows that under Assumption~\ref{assum: allrvalues}, gradient descent with small random initialization approaches the principal components $X_1,X_2,\ldots,X_r$ sequentially.

\begin{theorem}
\label{thm: mainthmallsigs}
Suppose that Assumption~\ref{assum: allrvalues} holds and let $\underline{\delta}$ be any number  in $(0,1)$ such that  $\underline{\delta} \le  \min_{1\le s \le r}\{\frac{\sigma_s - \sigma_{s+1}}{\sigma_s}\} $. Fix any tolerance $\epsilon\le \frac{1}{m+n+k}$. Then, there exist some numerical constants $c, c'$ and a sequence of iteration indices
    \[ T^{(1)} \le T^{(2)}\le \ldots T^{(r)} \le \frac{c'}{\underline{\delta}\eta \sigma_r} \log\left(\frac{\kappa_r}{\underline{\delta} \epsilon}\right) \] 
such that with high probability, the iterates of gradient descent \eqref{eq: GD} with stepsize $\eta \le c\min\{\underline{\delta}, 1-\underline{\delta}\}\frac{\sigma_r^2}{\sigma_1^3}$ and initialization size $\rho \le  \left(\frac{c  \underline{\delta}\epsilon}{ \kappa_r}\right)^{\frac{1}{c\underline{\delta}}}$ satisfy
    \[ \|F_{T^{(s)}}G_{T^{(s)}}^\top - X_s\| \le \epsilon\sigma_1,\quad \forall s = 1,2, \ldots, r. \]
\end{theorem}

Our next, more precise theorem applies under Assumption~\ref{assum: rthvalue} and shows that there is in fact a \emph{range} of iterations at which gradient descent approximates $X_r$. Note that the first theorem can be derived by applying the second theorem to each $r=1,2,\ldots$ 

\begin{theorem}
\label{thm: mainthmonesig}
Suppose that Assumption~\ref{assum: rthvalue} holds and let $\underline{\delta}$ be any number  in $(0,1)$ such that $\underline{\delta} \le \delta:= \frac{\sigma_r - \sigma_{r+1}}{\sigma_r}$.  Fix any tolerance $\epsilon\le \frac{1}{m+n+k}$. The following holds with high probability for some numerical constants $c, c'$. Consider gradient descent \eqref{eq: GD} with stepsize $\eta \le c\min\{\underline{\delta}, 1-\underline{\delta}\}\frac{\sigma_r^2}{\sigma_1^3}$ and initialization size $\rho 
\le \left(\frac{c  \underline{\delta}\epsilon}{ \kappa_r}\right)^{\frac{1}{c\underline{\delta}}}$. Let  $T = \Big\lfloor \frac{\log(\rho^{\frac{\underline{\delta}}{2(2-\underline{\delta})}}/\rho)}{\log(1+(1-\underline{\delta})\eta  \sigma_r)} \Big\rfloor$, which is $O\left( \frac{1}{\underline{\delta} \eta \sigma_r} \log\left(\frac{\kappa_r}{\underline{\delta}\epsilon}\right)\right)$ for
 $\rho 
=\left(\frac{c  \underline{\delta}\epsilon}{ \kappa_r}\right)^{\frac{1}{c\underline{\delta}}}$.
Then, for all $t$ such that $(1-c'\underline{\delta}) T \le t \le T$, we have 
	\begin{equation}\label{eq: equivalence}
			 \| F_{t}G_{t}^\top - X_r\| \le \epsilon\sigma_1.
			 \end{equation}
\end{theorem}

We highlight that Theorem~\ref{thm: mainthmonesig} applies to any observed matrix $X$ with a nonzero singular value gap $\delta$,\footnote{Otherwise $X_r$ is not uniquely defined.} which can be arbitrarily small; we refer to this as the \emph{model-free} setting.
Moreover, Theorem~\ref{thm: mainthmonesig} quantifies the relationship between various problem parameters: if $X$ has a relative singular value gap $\delta$, then gradient descent with initialization size $\rho \lesssim \epsilon^{\frac{1}{\delta}} $
and early stopping at iteration $O\left( \frac{1}{\delta } \log \frac{1}{\epsilon} \right)$ outputs $X_r$ up to an $\epsilon$ error (we ignore logarithmic terms).
 
We make two remarks regarding the tightness of the above parameter dependence.
\begin{itemize}
    \item \emph{Final error, initialization size, and eigen gap}: The final error $\epsilon$ can be made arbitrarily small as long as the initialization size $\rho$ is sufficiently small. Moreover, as verified by our numerical experiments in Section \ref{sec: numerics}, the scaling $\rho \lesssim \epsilon^{\frac{1}{\delta}} $ predicted by Theorem~\ref{thm: mainthmonesig} is quite accurate. 
    \item \emph{Iteration complexity and stepsize:} The number of iterations needed for an $\epsilon$ error scales as $\log(1/\epsilon)$, which is akin to a geometric/linear convergence rate. Moreover, our stepsize and iteration complexity are independent of the dimension $m,n$ (up to log factors), both of which improve upon existing results in~\cite{ye2021global}, which requires a significantly smaller stepsize and hence an iteration complexity proportional to $(m+n)^2$. Again, our dimension-independent scaling agrees well with the numerical results in Section~\ref{sec: numerics}.

\end{itemize}

\subsection{Proof Sketch for Theorem~\ref{thm: mainthmonesig}}
\label{sec: sketchAnalysis}

We sketch the main ideas for proving Theorem~\ref{thm: mainthmonesig} and discuss our main technical innovations. Our proof is inspired by the work~\cite{ye2021global}, which studies the setting with low-rank $X$ and exact parametrization $k=r=\text{rank}(X)$.

We start by simplifying the problem using the singular value decomposition (SVD) $X = \Phi \Sigma_X \Psi^\top$ of $X$, where $\Phi \in \RR^{m\times m}, \Sigma_X \in \RR^{m\times n}$ and $\Psi \in \RR^{n \times n}$. By replacing $F$, $G$ with $\Phi^\top F$, $\Psi^\top G$, respectively, we may assume without loss of generality that $X$ is diagonal. The distribution of the initial iterate $(F_0,G_0)$ remains the same thanks to the rotational invariance of Gaussian. Hence, the gradient descent update~\eqref{eq: GD} becomes
\begin{equation}\label{eq: GDdiag}
		F_+  = F + \eta(\Sigma_X - FG^\top) G,\quad \text{and}\quad 
		G_+  = G + \eta(\Sigma_X - FG^\top)^\top F, 
\end{equation}
where the subscript $+$ indicates the next iterate and will be used throughout the rest of the paper. 
Let $U$ be the upper $r\times k$ submatrix of $F$ and $J$ be the lower $(m-r)\times k$ submatrix of $F$. Similarly, let $V$ be the upper $r \times k$ submatrix of $G$ and $K$ be the lower $(n-r)\times k$ submatrix of $G$. Also let $\Sigma = \diag(\sigma_1,\ldots, \sigma_r)$ be the upper left $r\times r$  submatrix of and $\tilde \Sigma \in \RR^{(m-r)\times (n-r)}$ be a diagonal matrix with $\sigma_{r+1},\ldots, \sigma_{\min\{m,n\}}$ on the diagonal. The gradient descent update \eqref{eq: GDdiag} induces the following update for the ``signal'' matrices $U,V$ and the ``error'' matrices $J,K$:
\begin{align}\label{eqn: updateofUVJKintro}
	\begin{cases}
		U^+ = U +\eta \Sigma V - \eta U(V^\top V + K^\top K)\\
		V^+ = V+ \eta \Sigma U- \eta V(U^\top U + J^\top J)\\
		J^+ = J + \eta \tilde \Sigma K  - \eta J(V^\top V + K^\top K)\\
		K^+ = K + \eta \tilde \Sigma^\top J - \eta K(U^\top U + J^\top J).
	\end{cases}
\end{align}
We may bound the difference $FG^\top - X_r$ as the following:
\begin{align*}
	\|FG^\top - X_r\|\le \|UV^\top -\Sigma\| + \|UK^\top\| + \|JV^\top\| + \|JK^\top\|.
\end{align*}
Hence, it suffices to show that the signal term $UV^\top$ converges to $\Sigma$ and the error term $(J, K)$ remains small. To account for the potential imbalance of $U$ and $V$,  we introduce the following quantities using the symmetrization idea in~\cite{ye2021global}:
\begin{align}\label{eqn: decompositionintro}
	A = \frac{U+V}{2},\;
	B = \frac{U-V}{2},\;
	P= \Sigma -AA^\top + BB^\top, \;\text{and}\;
	Q= AB^\top -BA^\top.
	\end{align}
(Similarly define $A_t, B_t, P_t, Q_t$ based on the $t$-th iterates.) Here $A$ is the symmetrized part of the signal terms $U,V$, and $B$ represents the imbalance between them.
Since $\Sigma - UV^\top = P+Q$, the quantities $P$ and $Q$ capture how far the signal term is from the true signal $\Sigma$. 
Let $T$ be iteration index defined as in Theorem~\ref{thm: mainthmonesig}. Our proof studies three phases of lengths $T_1, T_2, T_3$ within these $T$ iterations, where $T_1 + T_2 + T_3 \le T$. The proof consists of three steps: 
\begin{itemize}
    \item \emph{Step 1.} We use induction on $t$ to show that the error term $(J_t,K_t)$ and the imbalance term $B_t$ remain small throughout the $T$ iterations. 

    \item \emph{Step 2.} We characterize the evolution of the smallest singular value $\sigma_r(A_t)$. After the first $T_1$ iteration, the value $\sigma_r(A_t)$ dominates the errors. Then, with $T_2$ more iterations, $\sigma_r(A_t)$ grows above the threshold $0.8\sqrt{\sigma_r}$, in which case signal term $UV^\top$ has magnitude close to that of the true signal $\Sigma_r$.
    \item \emph{Step 3.} After $\sigma_r(A_t)$ becomes sufficiently large, we show that $\|P_t\|$ decreases geometrically. After $T_3$ more iterations, $\|P_t\|$ has the same magnitude as the error terms $J_t, K_t$. Theorem~\ref{thm: mainthmonesig} then follows by bounding $\|F_tG_t^\top -X_r\|$ in terms of $P_t,Q_t,J_t,K_t$.
\end{itemize}

Our analysis of $P_t$ and $B_t$ departs from that in~\cite{ye2021global}, which requires the stepsize $\eta$ to depend on both problem dimension and initialization size, resulting in an iteration complexity that has polynomial dependence on the dimension. The better dependence in our Theorem~\ref{thm: mainthmonesig} is achieved using the following new techniques:
\begin{itemize}
    \item[1.] Unlike~\cite{ye2021global} which bounds ${B_t}, {J_t}, {K_t}$ by quantities independent of $t$, we control them using geometrically increasing series, which are tighter and more accurately capture the dynamics of ${B_t}, {J_t}, {K_t}$ across $t$.
    \item[2.] Our analysis decouples the choices of the stepsize $\eta$ and initialization size $\rho$, allowing them to be independent. We do so by making the crucial observation that the desired lower bound on $\lambda_r(P_t)$ only depends on $\sigma_r$ and the singular value gap $\delta$. As a result, we can take a very small initialization size $\rho$ (in both theory and experiments), since the iteration complexity depends only logarithmically on $\rho$. In contrast, the 
   desired lower bound on $\lambda_r(P_t)$ in \cite{ye2021global} depends on initialization size $\rho$,  and in turn their stepsize $\eta$ have more stringent dependence on $\rho$. 
    \item[3.] The analysis in~\cite{ye2021global} cannot be easily generalized to the overparametrized setting, since in this case the error terms $\|J_t\|, \|K_t\|$ no longer decay geometrically when they are within a small neighborhood of zero---to our best knowledge, tightly characterizing this local convergence behavior in the overparametrized setting remains an open problem. We circumvent this difficulty by using a smaller initialization size $\rho$, which is made possible by the tighter analysis outlined in the last bullet point.
\end{itemize}

\section{Experiments}
\label{sec: numerics}

We present numerical experiments that corroborate our theoretical findings on gradient descent with small initialization and early stopping. In Section~\ref{section: benefits}, we provide numerical results that demonstrate the dynamics and algorithmic regularization of gradient descent. In Section~\ref{sec: parameterDependence}, we numerically verify our theoretical prediction on the scaling relationship between the initialization size, stepsize, iteration complexity and final error. 

\subsection{Dynamics of gradient descent}\label{section: benefits}
 
We generate a random rank-$10$ matrix $X\in\mathbb{R}^{m\times n}$ with $m=250,n=200$ and $\fnorm{X}=1$, and run gradient descent~\eqref{eq: GD} with initialization size $\rho=10^{-6}$, stepsize $\eta=0.5$ and rank parameter $k=200$. As shown in Figure~\ref{fig: trainingvstestsingular}, the top singular values $\sigma_i(F_tG_t^\top)$ of the iterate grow towards those of $X$ sequentially at geometrically different rates. 
Consequently, for each $r=1,2,\ldots$, the iterate $F_tG_t^\top$ is close to the best rank-$r$ approximation $X_r$ when $\sigma_i(F_t G_t^\top) \approx \sigma_i(X), 1\le i\le r$ and $\sigma_i(F_t G_t^\top) \approx 0, i\ge r+1$; early stopping of gradient descent at this time outputs $X_r$.

 \begin{figure}[]
    \centering
    \subfloat[\label{fig: trainingvstestsingular}\footnotesize {The dashed lines track the evolution of singular values of the iterate $F_tG_t^\top$. The solid lines record the Frobenius norm distance between $F_tG_t^\top$ and $X_1, \ldots, X_4$, the leading components of $X$.
    }] 
        {\includegraphics[width=0.45\linewidth]{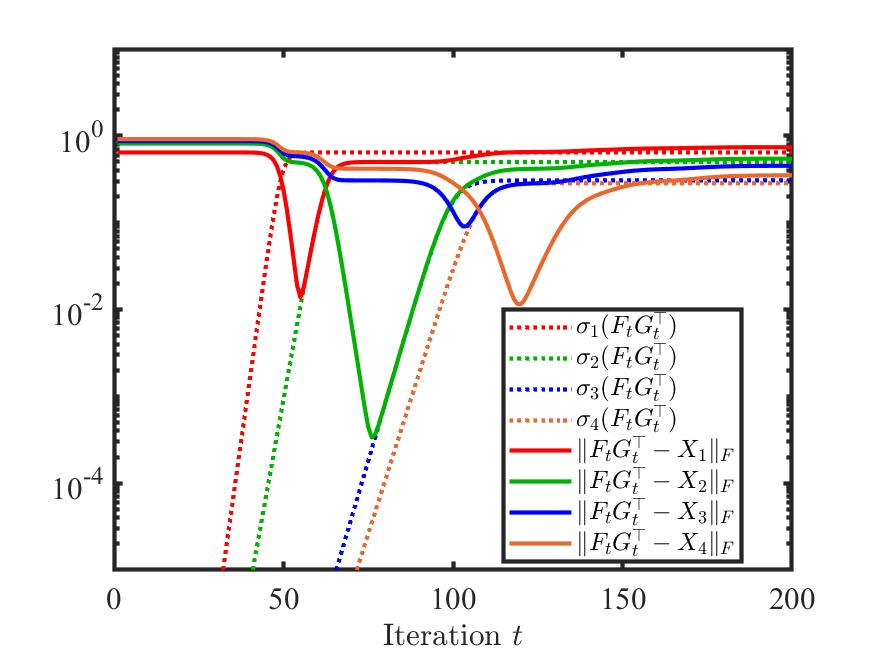}
} \hfill
      \subfloat[\label{fig: bigvssmall}\footnotesize {Convergence rate of  $\fnorm{F_tG_t^\top - X}$ for gradient descent with small ($\rho  \in\{10^{-4},10^{-6},10^{-8}\}\fnorm{X}$) and moderate ($\rho = \fnorm{X}$) initialization.
      }]  {\includegraphics[width=0.45\linewidth]{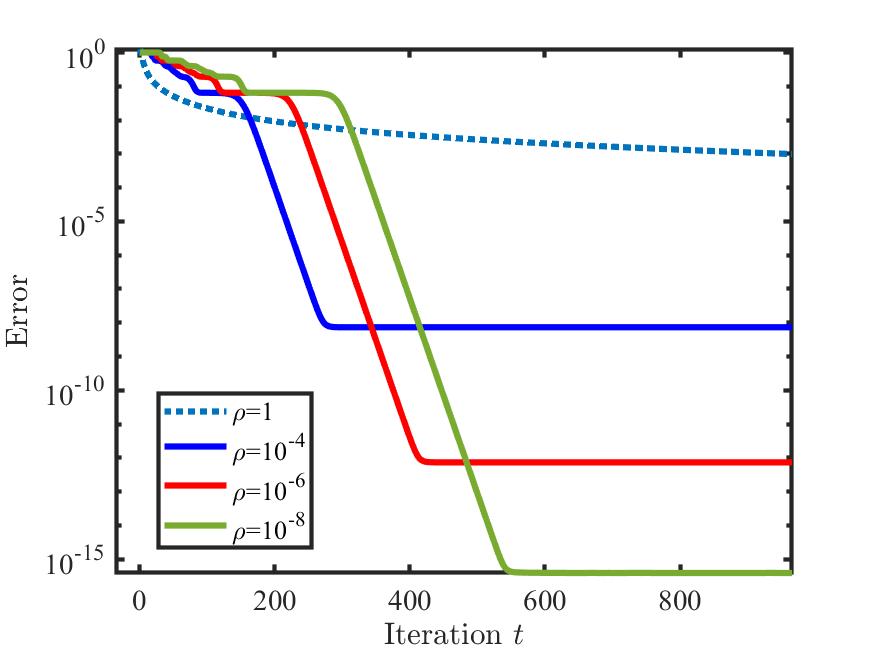}}
    \caption{\footnotesize {Dynamics of gradient descent with small random initialization and early stopping.}}
\end{figure}

In Figure~\ref{fig: bigvssmall}, we compare the convergence rates of gradient descent with small ($\rho \in \{10^{-2}, 10^{-4}, 10^{-6}\}$) and moderate ($\rho = 1$) initialization sizes. We see that using a small $\rho$ results in fast convergence to a small error level; moreover, the convergence rate is geometric-like (before saturation), which is consistent with the $\log(1/\epsilon)$ iteration complexity predicted by Theorem~\ref{thm: mainthmonesig}. Therefore, compared to moderate initialization, small initialization has both computational benefit (in terms of iteration complexity) and a statistical regularization effect (when coupled with early stopping, as demonstrated in Figure~\ref{fig: smGDeearlystop}).

\begin{figure}[t]
    \centering
    \subfloat[\footnotesize Relationship between $\epsilon$ and $\rho$\label{fig: FronormRelErrorInitialRho}]
        {\includegraphics[width=0.45\linewidth]{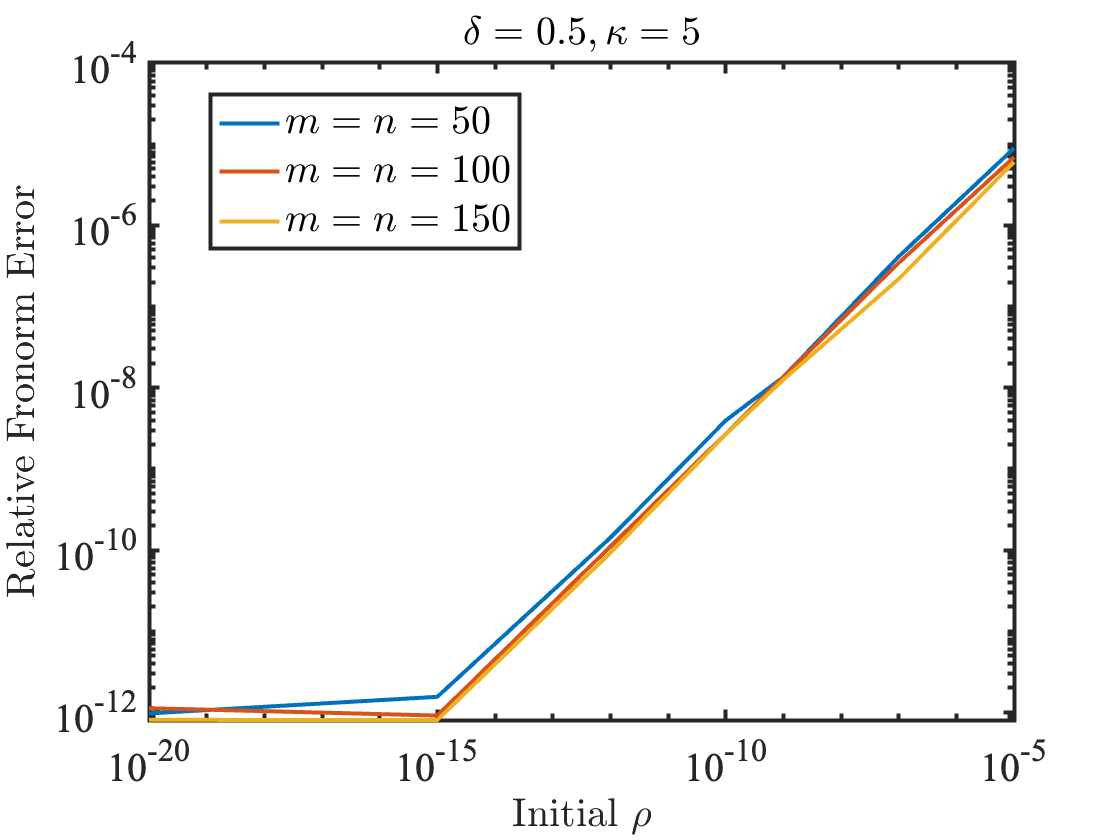}
} \hfill 
      \subfloat[{\footnotesize Relationship between $\epsilon$ and $\delta$}\label{fig: gapAndFronormRelError}]  {\includegraphics[width=0.45\linewidth]{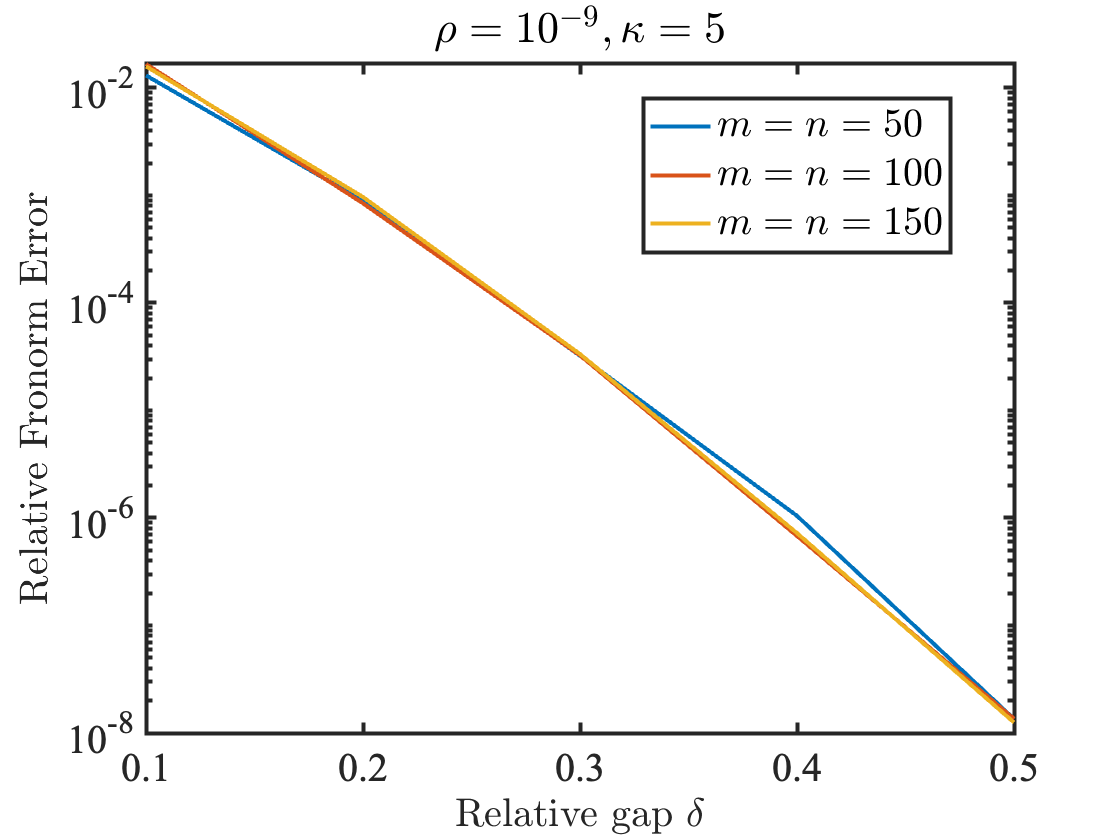}}
\hfill 
\subfloat[{\footnotesize Relationship between $T_0$ and $\delta$}\label{fig: gapAndT}]{
  \includegraphics[width=0.45\linewidth]{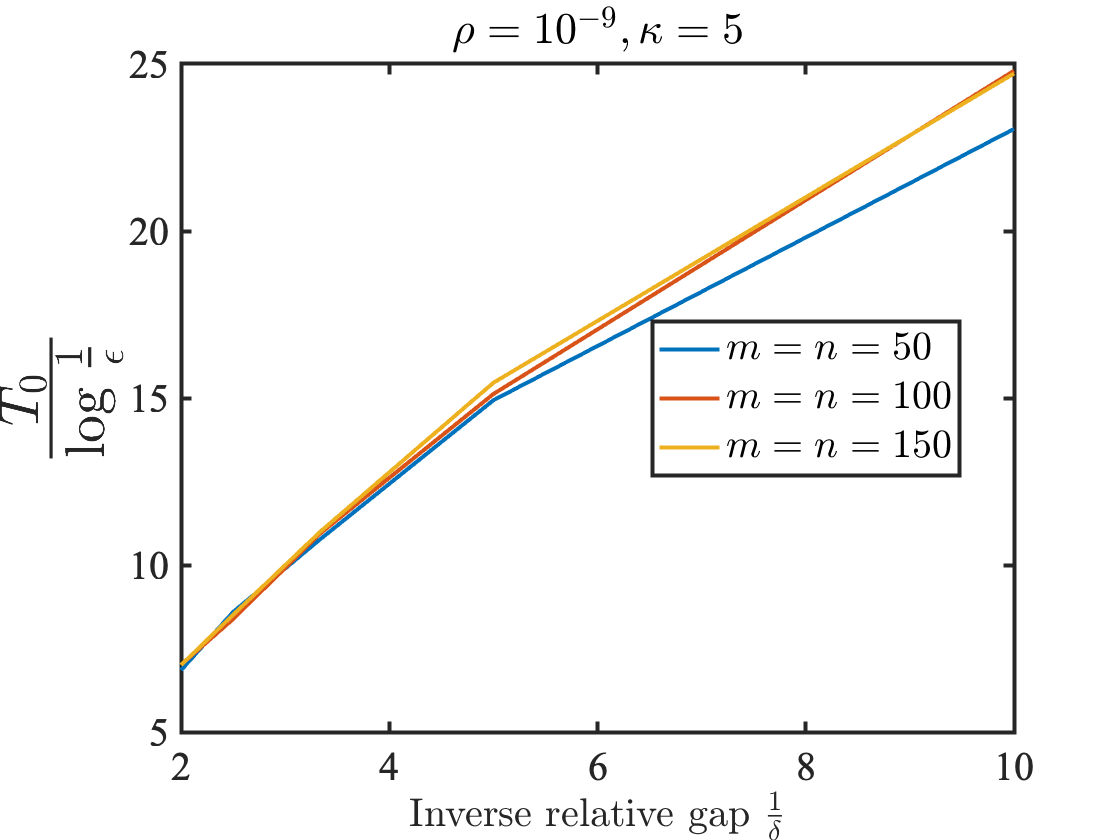}}
\hfill
\subfloat[\footnotesize $\epsilon$ vs.\ $(m,n)$ for fixed $\eta=0.25$\label{fig: stepsizeError.png}]{
  \includegraphics[width=0.45\linewidth]{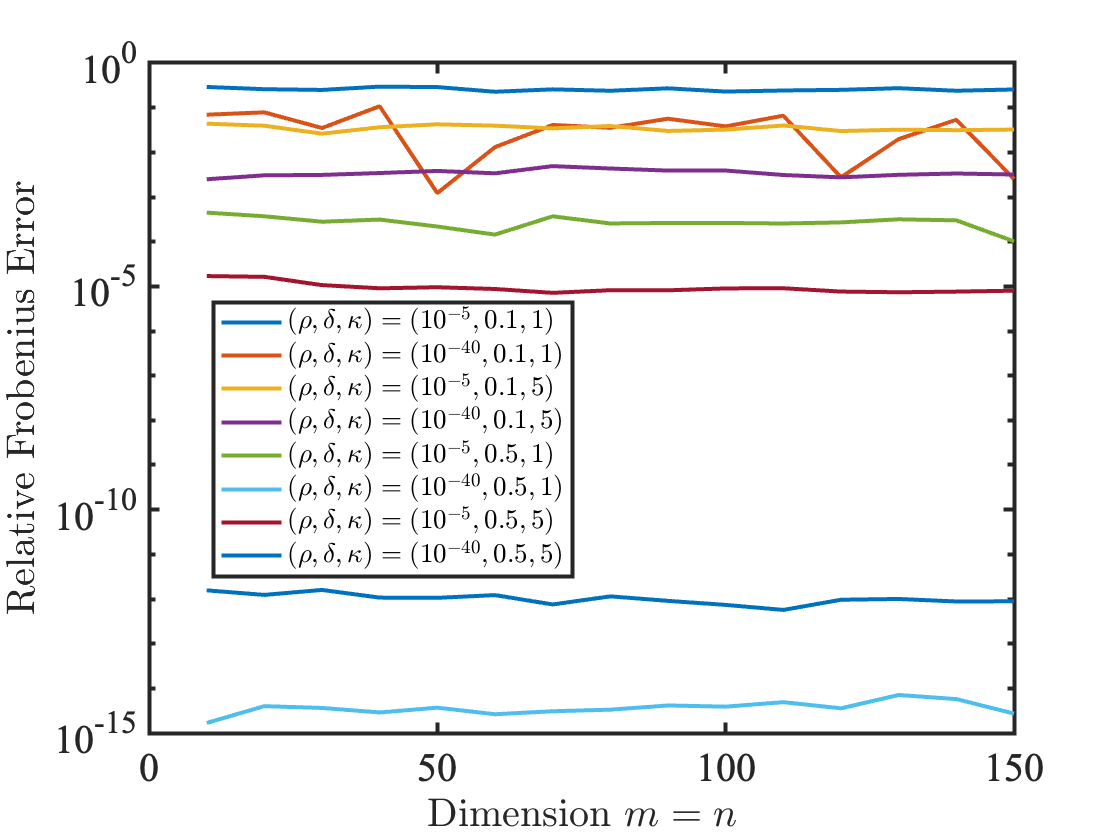}}

  \caption{Scaling relationship between relative Frobenius norm error $\epsilon$, initialization size $\rho$, relative singular value gap $\delta$, iteration complexity $T_0$ and stepsize $\eta$.}
 
\end{figure}

\subsection{Parameter dependence}
\label{sec: parameterDependence}

Theorem~\ref{thm: mainthmonesig} predicts that if $X$ has relative singular value gap  $\delta = \frac{\sigma_{r}-\sigma_{r+1}}{\sigma_r}$, then gradient descent outputs $X_r$ with a relative error $\fnorm{F_t G_t^\top-X_r}/\fnorm{X_r} = \epsilon$ in $T_0 = O\left( \frac{1}{\delta } \log \frac{1}{\epsilon} \right)$ iterations when using an initialization size $\rho \asymp \epsilon^{\frac{1}{\delta}} $ and a stepsize nearly independent of the dimension of $X$. We numerically verify these scaling relationships. 

In all experiments, we fix $r=3$, stepsize $\eta = 0.25$ and rank parametrization $k=10$, and generate a matrix $X\in \mathbb{R}^{m\times n}$ with $\sigma_3=1$. We vary the dimension $m=n$, relative singular value gap $\delta$ and condition number $\kappa = \frac{\sigma_1}{\sigma_r}$, and record the smallest relative error $\epsilon$ attained within $500$ iterations as well as the iteration index $T_0$ at which this error is attained. The results are the averaged over ten runs with randomly generated $X$ and shown in Figure~\ref{fig: FronormRelErrorInitialRho}-\ref{fig: stepsizeError.png}.

The results in Figure~\ref{fig: FronormRelErrorInitialRho} verify the relation $\log\epsilon \propto \log \rho$ for fixed $\delta$. We believe the flat part of the curves (for $\rho < 10^{-15}$) is due to numerical precision limits. Figure~\ref{fig: gapAndFronormRelError} verifies $\log \epsilon \propto -\delta$ for fixed $\rho$. Figure~\ref{fig: gapAndT} 
verifies $T_0 \propto \frac{1}{\delta}\log (\frac{1}{\epsilon})$. In all these plots, the the curves for different dimensions $m=n$ overlap, which is consistent with the (near) dimension-independent results in Theorem~\ref{thm: mainthmonesig}. Finally, Figure
\ref{fig: stepsizeError.png} shows that with a single fixed stepsize $\eta = 0.25$, the error $\epsilon$ is largely independent of the dimension $m=n$, for different values of $(\rho,\delta, \kappa)$. This supports the prediction of Theorem~\ref{thm: mainthmonesig} that the stepsize can be chosen independently of the dimension.

\section{Discussion}\label{sec: discussion}
In this paper, we characterize the dynamics of overparametrized vanilla gradient descent for asymmetric matrix factorization. We show that with sufficiently small random initialization and proper early stopping, gradient descent produces an iterate arbitrarily close to the best rank-$r$ approximation $X_r$ for any $r\leq k$ so long as the singular values $\sigma_r$ and $\sigma_{r+1}$ are distinct. Our theoretical results quantify the dependency between various problem parameters and match well with numerical experiments. Interesting future directions include extension to the matrix sensing/completion problems with asymmetric matrices, as well as understanding and capitalizing on the algorithmic regularization effect of overparametrized gradient descent in more complicated, nonlinear statistical models.

\section*{Acknowledgement}
Y.\ Chen is partially supported by National Science Foundation grants CCF-1704828 and  CCF-2047910. L.\ Ding is supported by National Science Foundation grant CCF-2023166.

\bibliographystyle{alpha}
\bibliography{reference}

\appendix

\section*{Organization of Appendix} 
We prove our main Theorem~\ref{thm: mainthmonesig} in Appendix~\ref{section: Analysis}, deferring some intermediate steps and technical lemmas to Appendices~\ref{section: induction} and~\ref{section: auxiliary_lemmas}, respectively.

\section{Proof of Theorem~\ref{thm: mainthmonesig}}
\label{section: Analysis}

For the ease of presentation, we work with (an upper bound) of the singular value ratio $ \gamma := 1-\underline{\delta} \ge \frac{\sigma_{r+1}}{\sigma_r}$, where we recall that $\underline{\delta} \le\frac{\sigma_r-\sigma_{r+1}}{\sigma_r}$ is a lower bound of the relative singular value gap.
Theorem~\ref{thm: mainthmonesig} is a simplified version of the following more precise theorem. The numerical constants below are not optimized.

\begin{theorem}
\label{thm: mainthmgeneralrestated}
Fix any $r\leq k $. Suppose $\sigma_{r+1}<\sigma_r$. Pick any $\gamma \in (0,1)$ such that $ \frac{ \sigma_{r+1}}{\sigma_r} \le \gamma$. Pick any stepsize $\eta \le \min\left\{\frac{\gamma\sigma_r^2}{600\sigma_1^3}, \frac{(1-\gamma)\sigma_r}{20\sigma_1^2}\right\}$. For any $c_\rho<1$, pick  any initialization size $\rho$ satisfying
	$$\rho \le \min\left\{\frac{1}{3}, \frac{1-\gamma}{24}, \frac{c_\rho \sqrt{\sigma_1}}{12(m+n+k) \sqrt{\frac{1-\gamma}{24}}\sqrt{\sigma_r} } \right\}^{\frac{180\gamma(1+\gamma)}{(3-2\gamma)(1-\gamma)}} $$ 
	and 
	$$\rho \le \min \left\{\left(\frac{(1-\gamma) c_\rho \sigma_r}{1200(m+n+k)r\sigma_1}\right)^{\frac{2(1+\gamma)}{1-\gamma}}, \left(\frac{\gamma \sigma_r^2}{1600r\sigma_1^2}\right)^{\frac{1+\gamma}{1-\gamma}}, \frac{\gamma \sigma_r \sqrt{2r}}{16\sigma_1 \sqrt{m+n+k}}\right\}.$$
Define
	\begin{equation}
	\begin{aligned}
	    T_1 &:=\floor{\frac{\log \left(\frac{12(m+n+k)\sqrt{\frac{1-\gamma}{24}}\sqrt{\sigma_r}}{c_\rho \size \sqrt{\sigma_1}}\right)}{\log(1+\frac{1+\gamma}{2}\eta \sigma_r)}}+1, 
	    \qquad 
	    T_2 := \floor{\frac{\log\left(\sqrt{\frac{24}{1-\gamma}}\right)}{\log(1+0.1\eta \sigma_r)}} +1,\\
	    T_3 &:= \floor{\frac{\log \left( \rho^{\frac{1-\gamma}{2(1+\gamma)}}/3\right)}{\log\left(1-\frac{3}{2}\eta \sigma_r \right)}}+1, \qquad 
	T = \floor{\frac{\log(\rho^{\frac{1-\gamma}{2(1+\gamma)}}/\rho)}{\log(1+\gamma\eta  \sigma_r)}}.
	\end{aligned}
	\end{equation}
Define $T_0:=T_1+T_2+T_3$. Then we have 
\begin{align}
    \frac{T_0}{T} \le 1- \frac{(3-2\gamma)(1-\gamma)}{6(3\gamma+1)}. \label{eqn: windowsize}
\end{align}
Moreover, there exists a universal constant $C$ such that with probability at least $1-(C\sizei)^{k-r+1} - C\exp(-k/C)$, for all $T_0\leq t\leq T$, we have
\begin{equation}\label{eq: Finalerror}
    \|F_tG_t^\top - X_r\| \le 8 \rho^{\frac{\delta}{2(2-\delta)}} \sigma_1 + 4\rho^{\frac{\delta}{2(2-\delta)}}\sqrt{2r}\sigma_1.
\end{equation}
\end{theorem}

\begin{proof}

The inequality~\eqref{eqn: windowsize} follows from the  auxiliary Lemma~\ref{lem: T1T2T3lessT}, which can be proved by mechanical though tedious calculation.

To prove the bound \eqref{eq: Finalerror}, we adopt the notations and strategy given in Section~\ref{sec: sketchAnalysis}, where we argue that it suffices to prove the bound in the case with $X= \Sigma_X$. 
Recall the quantities $A,B$ and $P$ defined in~\eqref{eqn: decompositionintro}. We can obtain the update of these quantities based on~\eqref{eqn: updateofUVJKintro}:
\begin{align}
	A_+  &= A + \eta P A - \eta(AB^\top - BA^\top)B - \eta A\frac{K^\top K+ J^\top J}{2}-  \eta B\frac{K^\top K- J^\top J}{2}; \label{eqn: updateofA} \\
	B_+  &= B - \eta PB + \eta(AB^\top - BA^\top)A - \eta A \frac{K^\top K- J^\top J}{2} - \eta B \frac{K^\top K + J^\top J}{2}\label{eqn: updateofB}
\end{align}
and
\begin{equation}
\label{eqn: updateofP}
\begin{aligned}
	P_+ &= P - \eta P(\Sigma - P) - \eta (\Sigma- P)P + \eta^2(PPP -P\Sigma P) - 2\eta BB^\top P - 2\eta PBB^\top \\
	 & \qquad - \eta(A+\eta PA)C^\top -\eta C(A+\eta PA)^\top - \eta^2 CC^\top  \\
	 & \qquad +\eta (B+\eta PB)D^\top +\eta D(B+\eta PB)^\top +\eta^2 DD^\top ,
\end{aligned}
\end{equation}
where 
\begin{align*}
C &= -AB^\top B + BA^\top B -A \frac{K^\top K+ J^\top J}{2} - B\frac{K^\top K - J^\top J}{2},\\
D &= AB^\top A - BA^\top A - A\frac{K^\top K - J^\top J}{2} - B\frac{K^\top K+J^\top J}{2}.
\end{align*}

Note that 
	\begin{align*}
		F_tG_t^\top - X_r = \begin{pmatrix}
			U_t \\
			J_t
		\end{pmatrix} \begin{pmatrix}
		V_t^\top & K_t^\top 
	\end{pmatrix} - \begin{pmatrix}
	\Sigma &  0 \\
	0    & 0
\end{pmatrix}= \begin{pmatrix}
	U_tV_t^\top - \Sigma & U_t K_t^\top \\
	J_tV_t^\top   & J_t K_t^\top
\end{pmatrix}.
	\end{align*}
Therefore, we have the bound
\begin{align*}
	\|F_tG_t^\top - X_r\|\le \|U_tV_t^\top -\Sigma\| + \|U_tK_t^\top\| + \|J_tV_t^\top\| + \|J_tK_t^\top\| .
\end{align*}
By Proposition~\ref{thm: inductivestep} and Proposition~\ref{thm: pgetsmaller} given in Section~\ref{section: induction} to follow, it holds with high probability that for any $T_1+T_2+T_3\le t\le T$, 
$$
\|U_tK_t^\top\| \le 3\rho^{\frac{1-\gamma}{2(1+\gamma)}}\sigma_1,\qquad \|J_tV_t^\top\| \le 3\rho^{\frac{1-\gamma}{2(1+\gamma)}}\sigma_1,\qquad \|J_tK_t^\top\|\le \rho^{\frac{1-\gamma}{(1+\gamma)}}\sigma_1,
$$
and 
\begin{align*}
	\|U_tV_t^\top -\Sigma\| &=\|P_t + Q_t\|\le \|P_t\| + \|Q_t\|\le \rho^{\frac{1-\gamma}{2(1+\gamma)}} \sigma_1 + 4\rho^{\frac{1-\gamma}{2(1+\gamma)}}\sqrt{2r}\sigma_1. 
\end{align*}
By combining pieces, we obtain that 
$$
\|F_tG_t^\top - X_r\| \le 8 \rho^{\frac{1-\gamma}{2(1+\gamma)}} \sigma_1 + 4\rho^{\frac{1-\gamma}{2(1+\gamma)}}\sqrt{2r}\sigma_1.
$$
thereby completing the proof of Theorem~\ref{thm: mainthmgeneralrestated}.
\end{proof}

\section{Induction steps for proving Theorem~\ref{thm: mainthmonesig}}
\label{section: induction}

\begin{proposition}[Base Case]\label{thm: basecase}
Suppose that  $\tilde F_0 \in \RR^{m\times k}$ and $ \tilde G_0 \in \RR^{n\times k}$ have i.i.d. $N(0,\sigma_1)$ entries. For any fixed $\size \le \sqrt{\frac{\sigma_r}{2\sigma_1}}$,  we take $F_0 = \frac{ \rho}{3\sqrt{m+n+k}} \tilde F_0$ and $G_0 = \frac{ \rho}{3\sqrt{m+n+k}}  \tilde G_0$. Then with probability at least $1-(C\sizei)^{k-r+1} - C\exp(-k/C)$, we have 
	\begin{enumerate}
		\item $\max\{\norm{U_0}, \norm{V_0}, \norm{J_0}, \norm{K_0}\} \le \rho \sqrt{\sigma_1}$.
		\item $\norm{U_0^\top U_0 +J_0^\top J_0} \le 7\sigma_1$.
		\item $\norm{V_0^\top V_0 + K_0^\top K_0} \le 7\sigma_1$.
		\item  $\sigma_r(A_0) \ge  \frac{\sizei \size \sqrt{\sigma_1}}{12(m+n+k)}$. 
		\item $\lambda_r(P_0) \ge 0$. \label{basecaseitem: 5}
		\item\label{basecaseitem: 6}$\fnorm{B_0}\le \rho\sqrt{2r} \sqrt{\sigma_1}$.
	\end{enumerate}
\end{proposition}

\begin{proof}
	The first four Items follow from Lemma~\ref{thm: initializationqualitysupple}. For Item~\ref{basecaseitem: 5}, by definition of $P_0$,
	\begin{align*}
		\lambda_r(P_0) \ge \sigma_r - \|A_0A_0^\top\| - \|B_0B_0^\top\|\ge \sigma_r - 2\rho^2 \sigma_1&\ge 0.
	\end{align*}
For Item~\ref{basecaseitem: 6}, by the fact that $\rank(B_0)\le 2r$, we have $\fnorm{B_0}\le \sqrt{2r} \norm{B_0} \le \rho \sqrt{2r}\sqrt{\sigma_1}.$
\end{proof}

In the sequel, we condition on the high probability event that the conclusion of Proposition~\ref{thm: basecase} holds. The rest of the analysis is deterministic.

\begin{proposition}[Inductive Step]\label{thm: inductivestep}
	Suppose that the stepsize satisfies $\eta \le \frac{\gamma\sigma_r^2}{600\sigma_1^3}$, 
	the initial size satisfies 
	$$\rho \le \max \left\{\left(\frac{(1-\gamma) c_\rho \sigma_r}{1200(m+n+k)r\sigma_1}\right)^{\frac{2(1+\gamma)}{1-\gamma}}, \left(\frac{\gamma \sigma_r^2}{1600r\sigma_1^2}\right)^{\frac{1+\gamma}{1-\gamma}}, \frac{\gamma \sigma_r \sqrt{2r}}{16\sigma_1 \sqrt{m+n+k}}\right\},
	$$
	and the following holds for all $0\le t \le s<T$ with $T\le  \floor{\frac{\log(\rho^{\frac{1-\gamma}{2(1+\gamma)}}/\rho)}{\log(1+\eta \gamma \sigma_r)}}$:
	\begin{enumerate}
		\item $\|U_t^\top U_t + J_t^\top J_t\| \le 7 \sigma_1$. \label{item 1}
		\item $\|V_t^\top V_t + K_t^\top K_t\| \le 7 \sigma_1$. \label{item 2}
		\item $\max\{\norm{J_t},\norm{K_t}\}\le (1+\gamma\eta \sigma_r)\max\{\norm{J_{t-1}}, \norm{K_{t-1}}\}\le \rho^{\frac{1-\gamma}{2(1+\gamma)}}\sqrt{\sigma_1}$, if $t \ge 1$. \label{item 3}
		\item $\lambda_r(P_t) \ge  \min\{(1-\eta \sigma_r)^2 \lambda_r(P_s)- 30\eta^2\sigma_1^3 -  \frac{\gamma\eta \sigma_r^2}{20},- 30\eta^2\sigma_1^3 -  \frac{\gamma\eta \sigma_r^2}{20}\} \ge -\frac{\gamma\sigma_r}{10}$, if $t \ge 1$. \label{item 4}
		\item\label{item 5} If $t \ge 1$, we have $\fnorm{B_{t}} \le \max \left\{ (1+\gamma\eta \sigma_r)\fnorm{B_{t-1}}, \frac{16\sigma_1 \sqrt{\sigma_1} \sqrt{m+n+k} (1+\gamma\eta \sigma_r)^{2t}\rho^2}{\gamma \sigma_r} \right\}$ and 
		$	\fnorm{B_{t}}	\le  (1+\gamma\eta \sigma_r)^{t} \rho \sqrt{2r}\sqrt{\sigma_1}
			\le \rho^{\frac{1-\gamma}{2(1+\gamma)}} \sqrt{2r}{\sqrt{\sigma_1}}\le \sqrt{\sigma_1}.$
	\end{enumerate}
Then the above items also hold for $t= s+1$. Consequently, by Proposition~\ref{thm: basecase} and induction, they hold for all $0\le t\le T$.
\end{proposition}

\begin{proof}
Assuming Item~\ref{item 1}--\ref{item 5} hold for all $t\in[0,s]$, we prove each item holds for $t=s+1$.
	
\textbf{Item~\ref{item 3}.} 
By~\eqref{eqn: updateofUVJKintro} and the induction hypothesis that $\norm{V_s^\top V_s+ K_s^\top K_s}\le 7\sigma_1$, We have
		\begin{align*}
		\norm{J_{s+1}} &\le \norm{J_s-\eta J_s(V_s^\top V_s + K_s^\top K_s)}+\norm{\eta \tilde \Sigma K_s}\\
		&\le \norm{J_s} + \gamma \eta \sigma_r\norm{K_s}\\
		&\le (1+\gamma\eta \sigma_r)\max\{\norm{J_s}, \norm{K_s}\}.
		\end{align*}
	By the same argument, we can show that $\norm{K_{s+1}} \le (1+\gamma\eta \sigma_r)\max\{\norm{J_s}, \norm{K_s}\}$, whence
	$$
	\max\{\norm{J_{s+1}},\norm{K_{s+1}}\} \le (1+\gamma\eta \sigma_r)\max\{\norm{J_s}, \norm{K_s}\}.
	$$
	By applying the above inequality inductively, we have 
	\begin{align*}
		\max\{\norm{J_{s+1}},\norm{K_{s+1}}\} &\le (1+\gamma\eta \sigma_r)^{s+1} \max\{\|J_0\|, \|K_0\|\}\\
		&\le (1+\gamma\eta \sigma_r)^T \max\{\|J_0\|, \|K_0\|\}\\
		&\le \rho^{\frac{1-\gamma}{2(1+\gamma)}}\sqrt{\sigma_1}.
	\end{align*}

\textbf{Item~\ref{item 4}.}  
By induction hypothesis and definition of $P_s$, we have $\lambda_r(\Sigma - A_sA_s^\top + B_sB_s^\top) \ge - \frac{\gamma\sigma_r}{10}$.
On the other hand, by induction hypothesis on Item~\ref{item 5}, we have $\|B_s\| \le \sqrt{\sigma_1}$. Therefore, we have $\|A_sA_s^\top\| \le 2\sigma_1$ and 
$$
\norm{P_s} \le \max\{ \norm{A_sA_s^\top},  \norm{\Sigma+ B_sB_s^\top}\} \le 2\sigma_1.
$$
By the updating formula of $P$ in~\eqref{eqn: updateofP}, we can write
\begin{align}\label{eqn: Pupdateappendix}
    P_{s+1} = (I - \eta(\Sigma -P_s))P_s(I-\eta(\Sigma-P_s))+E_s,
\end{align}
where 
\begin{align*}
	E_s &= -2\eta^2 \Sigma P_s\Sigma + \eta^2\Sigma P_s^2 + \eta^2 P_s^2 \Sigma- 2\eta B_sB_s^\top P_s - 2\eta P_sB_sB_s^\top - \eta(A_s+\eta P_sA_s)C_s^\top\\
 & \quad - \eta C_s(A_s+\eta P_sA_s)^\top - \eta^2 C_sC_s^\top +\eta (B_s+\eta P_sB_s)D_s^\top +\eta D_s(B_s+\eta P_sB_s)^\top \\
 & \quad + \eta^2 D_sD_s^\top.
\end{align*}
By $\|A_s\| \le \sqrt{2\sigma_1},\|B_s\| \le \sqrt{\sigma_1}$ and triangle inequality,  we have
\begin{align*}
	\|C_s\| &\le \|A_sB_s^\top B_s\| + \|B_s A_s^\top B_s\| + \|A_s \frac{K_s^\top K_s + J_s^\top J_s}{2}\|  + \|B_s\frac{K_s^\top K_s - J_s^\top J_s}{2}\|\\
	&\le 12 \rho^{\frac{1-\gamma}{1+\gamma}} r\sigma_1\sqrt{\sigma_1}.
\end{align*}
Similarly, we have $\|D_s\| \le 7\rho^{\frac{1-\gamma}{2(1+\gamma)}}\sqrt{2r} \sigma_1 \sqrt{\sigma_1}$.
By the bounds above, triangle inequality, and the upper bound on $\rho$, we have $$\|E_s\| \le 12 \eta^2 \sigma_1^3 +80\eta\rho^{\frac{1-\gamma}{1+\gamma}}r\sigma_1^2 \le 12\eta^2\sigma_1^3 + \frac{\gamma\eta \sigma_r^2}{20}.$$
Combining the above estimates on $\norm{E_s}$ and $\sigma_1(P)$, and the assumption on $\eta$, we see that we can apply Lemma~\ref{lem: sigmarP} to $P_{s+1}$ and obtain 
\begin{align*}
	\lambda_r(P_{s+1}) \ge \min\{(1-\eta \sigma_r)^2 \lambda_r(P_s)- 30\eta^2\sigma_1^3 - \frac{\gamma\eta \sigma_r^2}{20},- 30\eta^2\sigma_1^3 -  \frac{\gamma\eta \sigma_r^2}{20}\}.
\end{align*}
By the fact that $\lambda_r(P_0) \ge 0$, we have $\lambda_r(P_{s+1}) \ge -(30\eta^2 \sigma_1^3 +  \frac{\gamma\eta \sigma_r^2}{20}) \sum_{i=0}^{s}(1-\eta \sigma_r)^{2i}$. Hence,
\begin{align*}
	\lambda_r(P_{s+1}) &\ge  -(30\eta^2 \sigma_1^3 +  \frac{\gamma\eta \sigma_r^2}{20}) \sum_{i=0}^{\infty}(1-\eta \sigma_r)^{2i}\ge - \frac{30\eta^2 \sigma_1^3 + \frac{\gamma\eta \sigma_r^2}{20}}{\eta \sigma_r}\ge -\frac{\gamma\sigma_r}{10},
\end{align*}
where the last inequality follows from $\eta \le \frac{\gamma\sigma_r^2}{600\sigma_1^3}$.

\textbf{Item~\ref{item 5}.} 
Note that
\begin{align*}
	&\|B_{s+1}\|_F^2 \\
	&= \|B_s\|_F^2 - 2\eta \dotp{B_s B_s^\top, P_s } -\eta \fnorm{A_sB_s^\top - B_sA_s^\top}^2\\
	&\quad + \eta \dotp{A_s^\top B_s, K_s^\top K_s -J_s^\top J_s} - \eta \dotp{B_s^\top B_s,\frac{K_s^\top K_s + J_s^\top J_s}{2}}\\
	&\quad +\eta^2 \fnorm{-P_sB_s + (A_sB_s^\top - B_sA_s^\top)A_s- A_s\frac{K_s^\top K_s-J_s^\top J_s}{2} - B_s\frac{K_s^\top K_s+J_s^\top J_s}{2}}^2\\
	& \le \fnorm{B_s}^2 -2\eta \lambda_r(P_s)\fnorm{B_s}^2 + \eta \fnorm{B_s^\top B_s}\fnorm{K_s^\top K_s      + J_s^\top J_s}\\
		&\quad + \eta \fnorm{A_s^\top B_s}\fnorm{K_s^\top K_s - J_s^\top J_s}\\
		&\quad +\eta^2 \fnorm{-P_sB_s + (A_sB_s^\top - B_sA_s^\top)A_s- A_s\frac{K_s^\top K_s-J_s^\top J_s}{2} - B_s\frac{K_s^\top K_s+J_s^\top J_s}{2}}^2,
\end{align*}
where the equality follows from \eqref{eqn: updateofB} and brute force, and the inequality follows from Lemma~\ref{lem: traceineq}, Cauchy-Schwarz inequality and definition of $P_s$. 
By induction hypothesis, we have $\|K_s\|\le  \rho^{\frac{1-\gamma}{2(1+\gamma)}}\sqrt{\sigma_1}$ and $\|J_s\| \le  \rho^{\frac{1-\gamma}{2(1+\gamma)}}\sqrt{\sigma_1}$. Moreover, the rank of $K_s^\top K_s \pm J_s^\top J_s$ is at most $m+n+k$. Hence, 
$$
\fnorm{K_s^\top K_s \pm J_s^\top J_s} \le \sqrt{m+n+k} \norm{K_s^\top K_s\pm J_s^\top J_s} \le 2\sqrt{m+n+k}\rho^{\frac{1-\gamma}{1+\gamma}}\sigma_1\le {\frac{\gamma\sigma_r}{10}}.
$$
Furthermore, we have
\begin{align*}
	&\eta^2 \fnorm{-P_sB_s + (A_sB_s^\top - B_sA_s^\top)A_s- A_s\frac{K_s^\top K_s-J_s^\top J_s}{2} - B_s\frac{K_s^\top K_s+J_s^\top J_s}{2}}^2\\
	&\le 4\eta^2 \fnorm{P_sB_s}^2 + 4\eta^2 \fnorm{(A_sB_s^\top - B_sA_s^\top)A_s}^2+ 4\eta^2 \fnorm{ A_s\frac{K_s^\top K_s-J_s^\top J_s}{2}}^2\\
	& \quad + 4\eta^2 \fnorm{B_s\frac{K_s^\top K_s+J_s^\top J_s}{2}}^2\\
	&\le 48\eta^2 \sigma_1^2 \fnorm{B_s}^2 + 8\eta^2 \sigma_1^3(m+n+k)(1+\gamma \eta \sigma_r)^{4s} \rho^4 + 4\eta^2\sigma_1^2(1+\gamma \eta \sigma_r)^{4s} \rho^4\fnorm{B_s}^2\\
	&\le  (48\eta^2\sigma_1^2+ 4\eta^2\sigma_1^2(1+\gamma \eta \sigma_r)^{4s}\rho^4)\fnorm{B_s}^2 +8\eta^2 \sigma_1^3(m+n+k)(1+\gamma \eta \sigma_r)^{4s}\rho^4\\
	&\le \frac{\gamma\eta \sigma_r}{10} \fnorm{B_s}^2 + 8\eta^2 \sigma_1^3(m+n+k)(1+\gamma\eta \sigma_r)^{4s}\rho^4,
\end{align*}
where the second inequality follows from the induction hypothesis and the last inequality follows from the bound $\eta \le \frac{\gamma\sigma_r^2}{600\sigma_1^3}$. Combining with Item~\ref{item 4}, we have
\begin{align}
 \fnorm{B_{s+1}}^2 &\le \left(1+\frac{\gamma\eta \sigma_r}{2}\right)\fnorm{B_s}^2 + 4\eta \sigma_1\sqrt{\sigma_1} \sqrt{m+n+k}(1+\gamma\eta \sigma_r)^{2s} \rho^2 \fnorm{B_s} \notag \\
 &\quad + 8\eta^2 \sigma_1^3(m+n+k)(1+\gamma\eta \sigma_r)^{4s}\rho^4. \label{eqn: boundonB}
\end{align}
Consider two cases. Case (i):  $\fnorm{B_s}> \frac{16\sigma_1 \sqrt{\sigma_1} \sqrt{m+n+k} (1+\gamma\eta \sigma_r)^{2s}\rho^2}{\gamma \sigma_r}$. By~\eqref{eqn: boundonB}, we have 
	\begin{align*}
	\fnorm{B_{s+1}}^2 &\le  \left(1+\frac{\gamma\eta \sigma_r}{2}\right)\fnorm{B_s}^2 + \frac{\gamma\eta  \sigma_r}{4} \fnorm{B_s}^2+ \frac{\gamma^2\eta^2 \sigma_r^2 \fnorm{B_s}^2}{32} 
		\le (1+\gamma\eta \sigma_r)^2\fnorm{B_s}^2.
	\end{align*}
Hence, $\fnorm{B_{s+1}}\le \left(1+\gamma{\eta \sigma_r}{}\right)\fnorm{B_s}.$
Case (ii): $\fnorm{B_s}\le \frac{16\sigma_1 \sqrt{\sigma_1} \sqrt{m+n+k} (1+\gamma\eta \sigma_r)^{2s}\rho^2}{\gamma \sigma_r}$. By~\eqref{eqn: boundonB}, we have 
	\begin{align*}
		\fnorm{B_{s+1}}^2 &\le  \left(1+\frac{\gamma\eta \sigma_r}{2}\right)\frac{16^2 \sigma_1^3 (m+n+k) (1+\gamma \eta \sigma_r)^{4s}\rho^4}{\gamma^2\sigma_r^2}\\
		&\quad + \frac{4\gamma\eta \sigma_r}{16}\frac{16^2  \sigma_1^3 (m+n+k) (1+\gamma\eta \sigma_r)^{4s}\rho^4}{\gamma^2 \sigma_r^2}\\
		&\quad + \frac{8\gamma^2\eta^2\sigma_r^2}{16^2}\frac{16^2  \sigma_1^3 (m+n+k) (1+\gamma\eta \sigma_r)^{4s}\rho^4}{\gamma^2 \sigma_r^2}\\
		&\le \frac{16^2  \sigma_1^3 (m+n+k) (1+\gamma\eta \sigma_r)^{4(s+1)}\rho^4}{\gamma^2 \sigma_r^2}.
	\end{align*}
As a result, $\fnorm{B_{s+1}} \le \max \left\{ (1+\gamma\eta \sigma_r)\fnorm{B_{s}}, \frac{16\sigma_1 \sqrt{\sigma_1} \sqrt{m+n+k} (1+\gamma\eta \sigma_r)^{2(s+1)}\rho^2}{\gamma \sigma_r} \right\}.$

For the rest of Item~\ref{item 5}, we note that for any $0<t\le s+1$,
\begin{align*}
	&\max\left\{\fnorm{B_t}, \frac{16\sigma_1 \sqrt{\sigma_1} \sqrt{m+n+k} (1+\gamma\eta \sigma_r)^{2t}\rho^2}{\gamma \sigma_r}\right\}\\
	& \le (1+\gamma\eta \sigma_r)\max\left\{\fnorm{B_{t-1}},  \frac{16\sigma_1 \sqrt{\sigma_1} \sqrt{m+n+k} (1+\gamma\eta \sigma_r)^{2(t-1)}\rho^2}{\gamma \sigma_r}\right\}\\
	&\le (1+\gamma\eta \sigma_r)^t\max\left\{\fnorm{B_0},  \frac{16\sigma_1 \sqrt{\sigma_1} \sqrt{m+n+k}\rho^2}{\gamma \sigma_r}\right\}\\
	&\le (1+\gamma\eta \sigma_r)^t \rho \sqrt{2r}\sqrt{\sigma_1},
\end{align*}
where the last inequality follows from the base case that $\|B_0\|_F \le \rho \sqrt{2r}\sqrt{\sigma_1}$ and the assumption that $\rho \le \frac{\gamma \sigma_r \sqrt{2r}}{16\sigma_1 \sqrt{m+n+k}}$.
As a result,
$$
\fnorm{B_{s+1}}\le (1+\gamma\eta \sigma_r)^{s+1} \rho \sqrt{2r}\sqrt{\sigma_1} \le (1+\gamma\eta \sigma_r)^T \rho \sqrt{2r}\sqrt{\sigma_1}\le \rho^{\frac{1-\gamma}{2(1+\gamma)}} \sqrt{2r}{\sqrt{\sigma_1}}.
$$

\textbf{Items~\ref{item 1} and~\ref{item 2}.}
By the proof of Item~\ref{item 4}, we have 
$$
\lambda_r(\Sigma - A_{s+1}A_{s+1}^\top + B_{s+1}B_{s+1}^\top)=\lambda_r(P_{s+1}) \ge -\frac{\gamma\sigma_r}{10}.
$$
Therefore, we have $\|A_{s+1}\|\le \sqrt{\norm{A_{s+1}A_{s+1}^\top}}\le \sqrt{2\sigma_1}$. On the other hand, by the proof of Item~\ref{item 5}, we have $\|B_{s+1}\|\le \sqrt{\sigma_1}$. Therefore, we have
$$
\|U_{s+1}^\top U_{s+1} + V_{s+1}^\top V_{s+1}\| = \|2(A_{s+1}^\top A_{s+1} + B_{s+1}^\top B_{s+1})\| \le 6\sigma_1.
$$ 
This implies that $\|U_{s+1}^\top U_{s+1}\| \le 6\sigma_1$ and $\| V_{s+1}^\top V_{s+1}\|\le 6\sigma_1$. The result follows from triangle inequality and Item~\ref{item 3}.
\end{proof}
Next, we will show that $\sigma_r(A_t)$ will increase geometrically until it is at least $\sqrt{\frac{\sigma_r}{2}}$.

\begin{proposition}
\label{thm: sigmarAgetbigger}
Suppose that the conditions of Proposition~\ref{thm: inductivestep} holds. In addition, we assume $\eta \le  \frac{(1-\gamma)\sigma_r}{20\sigma_1^2}$. Then for any $0\le t \le T$, we have
	\[
	\sigma_r(A_t)\ge \min \bigg\{ \left(1+\frac{\gamma+1}{2}\eta \sigma_r\right)^t \sigma_r(A_0), \sqrt{\frac{1-\gamma}{24} \sigma_r} \bigg\}.
	\] 
	In particular, for $T_1 \!=\! \bigg\lfloor \frac{\log\! \Big(\frac{12(m+n+k)\sqrt{\frac{1-\gamma}{24}\sigma_r}}{c_\rho \size \sqrt{\sigma_1}}\Big)}{\log(1+\frac{1+\gamma}{2}\eta \sigma_r)} \bigg\rfloor \!+ 1$ and all $t\!\in[T_1, T]$,  we have 
	$\sigma_r(A_t) \ge \sqrt{\frac{1-\gamma}{24} \sigma_r}$.
\end{proposition}
\begin{proof}
	We prove it by induction. Clearly, the inequality holds for $t=0$. Suppose the result holds for $0 \le t<T$. 
	By the updating formula of $A$ in~\eqref{eqn: updateofA}, we have 
	\begin{align}\label{eqn: updateofAappendix}
		A_{t+1} = A_t +\eta(\Sigma -A_tA_t^\top)A_t +\eta E_t,
	\end{align}
where $E_t = B_tB_t^\top A_t - (A_tB_t^\top -B_tA_t^\top)B_t - A_t \frac{K_t^\top K_t +J_t^\top J_t}{2}-B_t\frac{K_t^\top K_t -J_t^\top J_t}{2}.$
Note that
\begin{align}
	\|E_t\| &\le  16(1+\gamma \eta \sigma_r)^{2t}\rho^2r\sigma_1\sqrt{\sigma_1}\label{eqn: boundonEappendix}\\
	&\le \frac{192(1+\gamma \eta \sigma_r)^t \rho (m+n+k)r \sigma_1}{c_\rho} \cdot \bigg(1+\frac{\gamma+1}{2}\eta \sigma_r\bigg)^t \frac{c_\rho \rho \sqrt{\sigma_1}}{12(m+n+k)}\notag \\
	&\le \frac{192 \rho^{\frac{1-\gamma}{2(1+\gamma)}}(m+n+k)r \sigma_1}{c_\rho}\cdot \bigg(1+\frac{\gamma+1}{2}\eta \sigma_r\bigg)^t \frac{c_\rho \rho \sqrt{\sigma_1}}{12(m+n+k)}\notag\\
	&\le \frac{1-\gamma}{6} \sigma_r\Bigg(1+\frac{\gamma+1}{2}\eta \sigma_r\bigg)^t \sigma_r(A_0)\notag,
\end{align}
where the first inequality follows from Proposition~\ref{thm: inductivestep} and triangle inequality, the second inequality from $\gamma \le \frac{\gamma+1}{2}$, the third inequality from $t\le T \le \Big\lfloor \frac{\log\big(\rho^{\frac{1-\gamma}{2(1+\gamma)}}/\rho\big)}{\log(1+\eta \gamma \sigma_r)} \Big\rfloor$, and the last inequality from the assumption  that $\rho \le\Big(\frac{(1-\gamma) c_\rho \sigma_r}{1200(m+n+k)r\sigma_1}\Big)^{\frac{2(1+\gamma)}{1-\gamma}}$.
On the other hand, 
\begin{align*}
    \|E_t\| &\le  16(1+\gamma \eta \sigma_r)^{2t}\rho^2r\sigma_1\sqrt{\sigma_1}
    \le 16\rho^{\frac{1-\gamma}{(1+\gamma)}}r\sigma_1\sqrt{\sigma_1}
    \le \frac{1-\gamma}{6} \sigma_r \sqrt{\frac{1-\gamma}{24}\sigma_r},
\end{align*}
where the last inequality follows from the assumption  that $\rho \le\Big(\frac{(1-\gamma) c_\rho \sigma_r}{1200(m+n+k)r\sigma_1}\Big)^{\frac{2(1+\gamma)}{1-\gamma}}$.
Combining both bounds on $\|E_t\|$, we have 
$$
\|E_t\| \le \frac{1-\gamma}{6}\sigma_r \min \left\{\left(1+\frac{\gamma+1}{2}\eta \sigma_r\right)^t \sigma_r(A_0), \sqrt{\frac{1-\gamma}{24}\sigma_r} \right\} \le  \frac{1-\gamma}{6}\sigma_r \sigma_r(A_t).
$$
Therefore, it holds that
\begin{align*}
	\sigma_r(A_{t+1}) &\ge \sigma_r(A_t +\eta(\Sigma -A_tA_t^\top)A_t)- \eta \norm{E_t}\\
	 &\ge (1-4\eta^2\sigma_1^2 )(1+\eta \sigma_r)\sigma_r(A_t)(1-\eta\sigma_r^2(A_t)) - \frac{1-\gamma}{6}\eta \sigma_r\sigma_r(A_t)\\
	 &\ge \left(1+ \frac{\gamma+3}{4} \eta \sigma_r\right)\sigma_r(A_t)(1- \eta \sigma_r^2(A_t)) - \frac{1-\gamma}{6}\eta \sigma_r\sigma_r(A_t),
\end{align*}
where the second inequality follows from Lemma~\ref{lem: sigmargetlarger} and the last inequality follows from the assumption that $\eta \le \min\left\{\frac{\gamma\sigma_r^2}{600\sigma_1^3}, \frac{(1-\gamma)\sigma_r}{20\sigma_1^2}\right\}$.
We consider two cases.
\begin{enumerate}
	\item $\sigma_r(A_t) \ge \sqrt{\frac{1-\gamma}{16}\sigma_r}$. We have
	\begin{align*}
		\sigma_r(A_{t+1}) 
	\ge (1 - \eta \sigma_1^2 (A_t)- \frac{1-\gamma}{6}\eta \sigma_r)\sigma_r(A_t) \sqrt{\frac{1-\gamma}{16}\sigma_r}
		\ge \sqrt{\frac{1-\gamma}{24}\sigma_r},
\end{align*}
where the last inequality follows from $\sigma_1(A_t) \le \sqrt{2\sigma_1}$ and $\eta \le \frac{1}{100 \sigma_1}$.
\item $\left(1+\frac{\gamma+1}{2} \eta \sigma_r\right)^t\sigma_r(A_0) \le \sigma_r(A_t) <\sqrt{\frac{1-\gamma}{16}\sigma_r}.$ Note that
\begin{align*}
	\sigma_r(A_{t+1}) & \ge (1+\frac{\gamma+3}{4}\eta \sigma_r - \frac{1-\gamma}{16} \eta  \sigma_r - \frac{1-\gamma}{16}\eta^2 \sigma_r^2 )\sigma_r(A_t)- \frac{1-\gamma}{6}\eta \sigma_r\sigma_r(A_t)\\
	&\ge (1+\frac{\gamma+1}{2} \sigma_r)\sigma_r(A_t)\\
	&\ge \min \left\{\left(1+\frac{\gamma +1 }{2}\eta \sigma_r\right)^{t+1} \sigma_r(A_0), \sqrt{\frac{1-\gamma}{24}\sigma_r} \right\}. 
\end{align*}
\end{enumerate}
Combining these two cases completes the induction step.
\end{proof}
After $\sigma_r(A_t)$ exceeds $\sqrt{\frac{1-\gamma}{24}\sigma_r}$, we show that $\sigma_r(A_t)$ increases to $0.8\sqrt{\sigma_r}$ at a slower rate.
\begin{proposition}\label{thm: growto0.9}
	Suppose that the conditions of Proposition~\ref{thm: inductivestep} holds. Let $T_1$ be the same from Proposition~\ref{thm: sigmarAgetbigger}. Then for any $T_1 \le t \le T$, we have 
	$$
	\sigma_r(A_t) \ge \min \bigg\{\left(1+ 0.1\eta \sigma_r\right)^{t-T_1} \sqrt{\frac{1-\gamma}{24}\sigma_r}, 0.8\sqrt{\sigma_r}\bigg\}.
	$$
	In particular,  for $T_2 = \bigg\lfloor \frac{\log\left(\sqrt{\frac{24}{1-\gamma}}\right)}{\log(1+0.1\eta \sigma_r)} \bigg\rfloor +1$ and all $T_1+T_2 \le t \le T$, we have $
	\sigma_r(A_t) \ge 0.8 \sqrt{\sigma_r}$.
\end{proposition}
\begin{proof}
	We prove it by induction. By Proposition~\ref{thm: sigmarAgetbigger}, the inequality holds for $t=T_1$. Suppose the result holds for $T_1 \le t <T$. By the same argument of Proposition~\ref{thm: sigmarAgetbigger}, \eqref{eqn: updateofAappendix} and \eqref{eqn: boundonEappendix} hold. It follows that
    \begin{equation}
        \|E_t\|\le 16(1+\gamma \eta \sigma_r)^{2t} \rho^2 r\sigma_1 \sqrt{\sigma_1}
    	\le 16 \rho^{\frac{1-\gamma}{1+\gamma}} r\sigma_1 \sqrt{\sigma_1}
    	\le 0.01\sigma_r \sqrt{\frac{1-\gamma}{24}\sigma_r}
    	\le 0.01\sigma_r \sigma_r(A_t).
    \end{equation}
Applying Lemma~\ref{lem: sigmargetlarger}, we have 
\begin{align*}
	\sigma_r(A_{t+1}) &\ge \sigma_r(A_t +\eta(\Sigma -A_tA_t^\top)A_t)- \eta \|E_t\| \\
	&\ge  (1-4\eta^2\sigma_1^2 )(1+\eta \sigma_r)\sigma_r(A_t)(1-\eta\sigma_r^2(A_t)) - 0.01\eta \sigma_r\sigma_r(A_t)\\
	&\ge (1+0.99 \eta \sigma_r) \sigma_r(A_t)(1-\eta \sigma_r^2(A_t)) - 0.01\eta \sigma_r \sigma_r(A_t)
\end{align*}
We consider two cases.
\begin{enumerate}
	\item $\sigma_r(A_t) \ge 0.9\sqrt{\sigma_r}$. We have
	\begin{align*}
		\sigma_r(A_{t+1})
		\ge (1 - \eta \sigma_1^2 (A_t)- 0.01\eta \sigma_r)\sigma_r(A_t)
		\ge 0.8\sqrt{\sigma_r},
	\end{align*}
where the last inequality follows from $\sigma_1(A_t) \le \sqrt{2\sigma_1}$ and $\eta \le \frac{1}{100 \sigma_1}$.
	\item $\left(1+ 0.1\eta \sigma_r\right)^{t-T_1} \sqrt{\frac{1-\gamma}{24}\sigma_r} \le \sigma_r(A_t) <0.9\sqrt{\sigma_r}.$ We have
	\begin{align*}
		\sigma_r(A_{t+1}) & \ge (1+0.99\eta \sigma_r)(1-0.81\eta \sigma_r)\sigma_r(A_t)- 0.01\eta \sigma_r\sigma_r(A_t)\\
		&\ge (1+0.1\eta \sigma_r)\sigma_r(A_t)\\
		&\ge \min \bigg\{ \left(1+ 0.1\eta \sigma_r\right)^{t+1-T_1} \sqrt{\frac{1-\gamma}{24}\sigma_r}, 0.8\sqrt{\sigma_r} \bigg\}.
	\end{align*}
\end{enumerate}
Combining the two cases completes the induction step.
\end{proof}

After $\sigma_r(A_t)$ exceeds $0.8\sqrt{\sigma_r}$, we show that $\|P_t\|$ decreases geometrically.
\begin{proposition}\label{thm: pgetsmaller}
	Let $T_1$ and $T_2$ be defined as Proposition~\ref{thm: sigmarAgetbigger} and Proposition~\ref{thm: growto0.9}. Suppose that conditions of Proposition~\ref{thm: inductivestep} hold.  For any $ T_1+T_2< t \le T$,  we have 
	$$
	\|P_t\|\le2\left(1-\frac{3\eta \sigma_r}{2}\right)^{t-T_1-T_2}\sigma_1 + \frac{80\rho^{\frac{1-\gamma}{1+\gamma}}r\sigma_1^2}{\sigma_r}.
	$$
    Let $T_3= \bigg\lfloor \frac{\log \big( \rho^{\frac{1-\gamma}{2(1+\gamma)}}/3\big)}{\log\left(1-\frac{3}{2}\eta \sigma_r \right)} \bigg\rfloor+1$. Then for any $T_1+T_2+T_3\le t \le T$, we have $\|P_t\| \le \rho^{\frac{1-\gamma}{2(1+\gamma)}}\sigma_1$.
	 
\end{proposition}
\begin{proof}
	Recall the updating formula~\eqref{eqn: Pupdateappendix}.
 By the same argument as the proof of the item ~\ref{item 4} in Proposition~\ref{thm: inductivestep}, for $T_1+T_2 \le t < T$, we have the bound $\|E_t\|\le  6\eta^2 \sigma_1^2 \|P_t\|+80\eta\rho^{\frac{1-\gamma}{1+\gamma}}r\sigma_1^2 $. On the other hand,
\begin{align*}
	(I - \eta(\Sigma -P_t))P_t(I-\eta(\Sigma-P_t)) &= (I - \eta A_tA_t^\top + \eta B_tB_t^\top)P_t(I - \eta A_tA_t^\top + \eta B_tB_t^\top).
\end{align*}
By Proposition~\ref{thm: inductivestep} and Proposition~\ref{thm: growto0.9}, we have $\sigma_r(A_t) \ge 0.8\sqrt{\sigma_r}$ and $\|B_t\|\le \fnorm{B_t}\le 0.1\sqrt{\sigma_r}$. It follows that 
$$
\|I - \eta A_tA_t^\top + \eta B_tB_t^\top\| \le 1- 0.8\eta \sigma_r + 0.01\eta \sigma_r\le 1-0.79\eta \sigma_r , 
$$
hence
$$
\|(I - \eta(\Sigma -P_t))P_t(I-\eta(\Sigma-P_t))\| \le (1-0.79\eta \sigma_r)^2 \norm{P_t}.
$$
As a result, we have
\begin{align*}
	\|P_{t+1}\| &\le (1-0.79\eta \sigma_r)^2 \norm{P_t} + 6\eta^2 \sigma_1^2\|P_t\| +80\eta\rho^{\frac{1-\gamma}{1+\gamma}}r\sigma_1^2 \\
	&\le (1-\frac{3\eta \sigma_r}{2})\|P_t\| +80\eta\rho^{\frac{1-\gamma}{1+\gamma}}r\sigma_1^2 .
\end{align*}
Applying the above inequality inductively, we have
\begin{align*}
	\|P_{t+1}\| &\le \left(1-\frac{3\eta \sigma_r}{2}\right)^{t+1-T_1-T_2} \|P_{T_1+T_2}\| + \sum_{i=0}^{\infty} \left(1-\frac{3\eta \sigma_r}{2}\right)^i \cdot 80\eta\rho^{\frac{1-\gamma}{1+\gamma}}r\sigma_1^2\\
	&\le 2\left(1-\frac{3\eta \sigma_r}{2}\right)^{t+1-T_1-T_2}\sigma_1 + \frac{80\rho^{\frac{1-\gamma}{1+\gamma}}r\sigma_1^2}{\sigma_r}.
\end{align*}
The result follows if we shift the index of $P$ by one. The statement for $T_3$ follows from direct calculation and the upper bound on $\rho$.
\end{proof}

\section{Auxiliary lemmas and additional literature review}
\label{section: auxiliary_lemmas}
\subsection{Auxiliary lemmas}
In this section, we state several technical lemmas that are used in the proof of Theorem~\ref{thm: mainthmonesig}. These lemmas are either direct consequence of standard results or follow from simple algebraic calculation. 

\begin{lemma}[{\cite[Theorem 1.1]{rudelson2009smallest}}]
\label{thm: smallest s.v. initializationsupp}
	Let $A$ be an $r \times k$ matrix with $r \le k$ and i.i.d.\ Gaussian entries with distribution $N(0,1)$. Then for every $\sizei>0$ we have with probability at least $1-(C\sizei)^{k-r+1} - \exp(-k/C)$ that 
	$$
	\sigma_{\min} (A) \ge \sizei\left(\sqrt{k} - \sqrt{r-1}\right)\ge \frac{\sizei}{2\sqrt{k}},
	$$
	where $C$ is a universal constant.
\end{lemma}

\begin{lemma}[{\cite[Example 2.32, Exercise 5.14]{wainwright2019high}}]
\label{thm: largest s.v. initializationsupp}
	Let $W$ be an $n \times N$ matrix with i.i.d. Gaussian entries with distribution $N(0,1)$. Then there exists a universal constant $C$ such that 
	$$
	\sigma_{\max}(W) \le 3\sqrt{n+N}
	$$
	with probability at least $1- 2e^{-\frac{n+N}{2}}$.
\end{lemma}

\begin{lemma}[Initialization Quality]\label{thm: initializationqualitysupple}
	Suppose that we sample $\tilde F_0 \in \RR^{m\times k}, \tilde G_0 \in \RR^{n\times k}$ with i.i.d.\ $N(0,\sigma_1)$ entries. For any fixed $\size>0$,  if we take $F_0 = \frac{ \rho}{3\sqrt{m+n+k}} \tilde F_0$ and $G_0 = \frac{ \rho}{3\sqrt{m+n+k}}  \tilde G_0$, then with probability at least $1-(C\sizei)^{k-r+1} - C\exp(-k/C)$, we have 
	\begin{enumerate}
		\item $\norm{F_0} \le \size \sqrt{\sigma_1}$.
		\item $\norm{G_0} \le \size \sqrt{\sigma_1}$.
		\item $\sigma_r(A_0) \ge  \frac{\sizei \size \sqrt{\sigma_1}}{12(m+n+k)}$. 
	\end{enumerate}
\end{lemma}

\begin{proof}
	By Proposition~\ref{thm: largest s.v. initializationsupp}, with probability at least  $1- 2e^{-\frac{m+k}{2}}-2e^{-\frac{n+k}{2}}$, we have $\|\tilde F_0\|\le 3\sqrt{m+k}\sqrt{\sigma_1}$ and $\|\tilde G_0\|\le 3\sqrt{n+k}\sqrt{\sigma_1}$. Hence, we have $\|F_0\|\le \rho\sqrt{\sigma_1}$ and $\|G_0\|\le \rho \sqrt{\sigma_1}$. On the other hand, $A_0 = \frac{U_0+V_0}{2}$ has i.i.d. $N(0, \frac{\sigma_1}{2})$ entries. By Proposition~\ref{thm: smallest s.v. initializationsupp}, with probability at least $1-(C\sizei)^{k-r+1} - \exp(-k/C)$, we have 
	$$
	\sigma_r(A_0) =\frac{\rho}{3\sqrt{m+n+k}}\sigma_r(\tilde A_0) \ge \frac{\sizei\rho\sqrt{\sigma_1}}{12(m+n+k)}.
	$$
	Combining, all items hold with probability at least 
	$$1-(C\sizei)^{k-r+1} - \exp(-k/C)-2e^{-\frac{m+k}{2}}-2e^{-\frac{n+k}{2}}.$$ 
	By increasing $C$ if necessary, we obtain the desired result.
\end{proof}

\begin{lemma}\label{lem: S(I-SS) svd}
	Let $S$ be an $r\times k$ matrix with $r\le k$. If $\|S\| \le \sqrt{\frac{1}{3\eta}}$, then the largest and smallest singular values of $S(I-\eta S^\top S)$ are $\sigma_1(S) - \eta \sigma_1^3(S)$ and $\sigma_r(S) - \eta \sigma_r^3(S)$, respectively.
\end{lemma}
\begin{proof}
	Let $U_S \Sigma_S V_S^\top$ be the SVD of $S$. Simple algebra shows that
	\begin{equation}
		S\left(I - \eta S^\top S\right) = U_S\left(\Sigma_S - \eta\Sigma_S^3\right) V_S^\top.
	\end{equation}
	This is exactly the SVD of $S\left(I - \eta S^\top S\right)$. Let $g(x) = x - \eta x^3$. Ty taking derivative, $g$ is monotone increasing in interval $[-\sqrt{\frac{1}{3\eta}}, \sqrt{\frac{1}{3\eta}}]$. Since the singular values of $S(I-\eta S^\top S)$ are exactly the singular values of $S$ mapped by $g$, the result follows.
\end{proof}
\begin{lemma}\label{lem: norm of inverse}
	Let $S$ be an $r\times r$ matrix such that $\norm{S} <1$. Then $I +S$ is invertible and 
	$
		\norm{(I+S)^{-1}} \le \frac{1}{1-\norm{S}}.
	$
\end{lemma}
\begin{proof}
	Since $\norm{S}<1$, the matrix $T= \sum_{i=0}^{\infty}(-1)^iS^i$ is well defined and indeed $T$ is the inverse of $I+S$. By continuity, subadditivity and submultiplicativity of operator norm,
	\begin{align}
		\norm{(I+S)^{-1}} = \norm{T} \le \sum_{i=0}^{\infty} \norm{S^{i}} \le \sum_{i=0}^{\infty} \norm{S}^i = \frac{1}{1-\norm{S}}.
	\end{align}
\end{proof}
\begin{lemma}\label{lem: sigmar ineq}
	Let $S$ be an $r\times r$ matrix and $T$ be an $r\times k$ matrix. Then
	$
		\sigma_r(ST) \le \norm{S} \sigma_r(T).
	$
\end{lemma}
\begin{proof}
	For any $r \times k$ matrix $R$, the variational expression of $r$-th singular value is 
	\begin{equation}
		\sigma_r(R) = \sup_{\substack{\text{subspace } M \subset \RR^k\\ \dim(M)=r}}\inf_{\substack{x \in M\\ x\neq 0}} \frac{\norm{Rx}}{\norm{x}}. 
	\end{equation}
	Applying this variational result twice, we obtain
	\begin{align}
		\sigma_r(ST)&= \sup_{\substack{\text{subspace } M \subset \RR^k\\ \dim(M)=r}}\inf_{\substack{x \in M\\ x\neq 0}} \frac{\norm{STx}}{\norm{x}}\\
		&\le  \sup_{\substack{\text{subspace } M \subset \RR^k\\ \dim(M)=r}}\inf_{\substack{x \in M\\ x\neq 0}} \frac{\norm{S}\norm{Tx}}{\norm{x}}\\
		&= \norm{S}\sigma_r(T).
	\end{align}
\end{proof}

\begin{lemma}\label{lem: sigmargetlarger}
	Let $S$ be an $r \times k$ matrix with $r \le k$. Suppose that $\sigma_r(S)>0$. Let $\Lambda =\diag \{\sigma_1,\ldots, \sigma_r\} \in \RR^{r\times r}$ be a diagonal matrix. Suppose that  $\eta \|\Lambda -SS^\top\| <\frac{1}{2}$, $\|S\|\le \sqrt{\frac{1}{3\eta}}$, $2\eta^2 \norm{\Lambda S S^\top}<1$,  and 
	$
	S_+ = S + \eta(\Lambda - SS^\top)S.
	$
	Then we have $\sigma_r(S_+) \ge (1-2\eta^2 \norm{\Lambda S S^\top})(1+\eta \sigma_r)\sigma_r(S)(1-\eta\sigma_r^2(S))$.
\end{lemma}
\begin{proof}
	Since $\eta \|\Lambda -SS^\top\| <1$, matrix $I + \eta (\Lambda -SS^\top)$ is invertible. Hence, we can write 
	$$
	S= (I + \eta (\Lambda -SS^\top))^{-1} S_+.
	$$
	On the other hand, by definition of $S_+$, we have
	\begin{align*}
	S_+ &= (I+\eta \Lambda) S(I- \eta S^\top S) + \eta^2 \Lambda SS^\top S\\
	&= (I+\eta \Lambda) S(I- \eta S^\top S) + \eta^2 \Lambda SS^\top(I + \eta  (\Lambda -SS^\top))^{-1} S_+.
	\end{align*}
Therefore, 
$$
(I- \eta^2 \Lambda SS^\top(I + \eta (\Lambda -SS^\top))^{-1})S_+ = (I+\eta \Lambda) S(I- \eta S^\top S) 
$$
By Lemma~\ref{lem: S(I-SS) svd}, we have 
\begin{align*}
	\sigma_r((I+\eta \Lambda)S(I-\eta S^\top S)) &\ge (1+\eta \sigma_r) \sigma_r(S(I-\eta S^\top S))\\
	& = (1+\eta \sigma_r)\sigma_r(S)(1-\eta \sigma_r(S)^2).
\end{align*}
On the other hand, by Lemma~\ref{lem: norm of inverse},
\begin{align*}
	&\sigma_r((I- \eta^2 \Lambda SS^\top(I + \eta (\Lambda -SS^\top))^{-1})S_+)\\
	&\le \|I- \eta^2 \Lambda SS^\top(I + \eta (\Lambda -SS^\top))^{-1}\| \sigma_r(S_+)\\
	&\le \Big(1+ \eta^2 \frac{\norm{\Lambda S S^\top}}{1-\eta \norm{\Lambda -SS^\top}}\Big) \sigma_r(S_+)\\
	&\le \Big(1+2\eta^2 \norm{\Lambda S S^\top}\Big)\sigma_r(S_+)
\end{align*}
Combining, we have
\begin{align*}
	\sigma_r(S_+) &\ge  \frac{(1+\eta \sigma_r)\sigma_r(S)(1-\eta\sigma_r^2(S))}{1+2\eta^2 \norm{\Lambda S S^\top}}\\
	& \ge \Big(1-2\eta^2 \norm{\Lambda S S^\top}\Big)(1+\eta \sigma_r)\sigma_r(S)(1-\eta\sigma_r^2(S)).
\end{align*}
\end{proof}
\begin{lemma}\label{lem: traceineq}
    Let $B \in \RR^{r\times k}$ be a real matrix and $P \in \RR^{r \times r}$ a symmetric matrix. We have 
    $
    \dotp{BB^\top, P} \ge \lambda_r(P) \|B\|_F^2.
    $
\end{lemma}
\begin{proof}
    Let the eigenvalue decomposition of $P$ be $P = U^\top \Lambda U$, where $U \in \RR^{r\times r}$ is an orthogonal matrix and $\Lambda = \diag\{\lambda_1, \ldots, \lambda_r\}$ consists of eigenvalues of $P$.  Let $B' = UB$, we have
    \begin{align*}
        \dotp{BB^\top, P} &= \text{trace}(BB^\top P)\\
        &= \text{trace}(BB^\top U^\top \Lambda U)\\
        &= \text{trace}(B'^\top \Lambda B')\\
        &\ge \text{trace}(B'^\top (\Lambda- \lambda_r I) B') + \lambda_r \text{trace}(B'^\top I B')\\
        &\ge \lambda_r \text{trace}(B'^\top I B')\\
        &= \lambda_r\|B'\|_F^2 \\
        &= \lambda_r \|B\|_F^2,
    \end{align*}
    where the last inequality follows from the fact that $B'^\top (\Lambda- \lambda_r I) B'$ is positive semi-definite.
\end{proof}
\begin{lemma}\label{lem: T1T2T3lessT}
    Under the assumption of Theorem~\ref{thm: mainthmgeneralrestated}, we have 
    $$
    \frac{T_1+T_2+T_3}{T} \le 1- \frac{(3-2\gamma)(1-\gamma)}{6(3\gamma+1)}.
    $$
\end{lemma}
\begin{proof}
    For simplicity, we omit the flooring and ceiling operations and assume 
	$$T_1 =\frac{\log \bigg(\frac{12(m+n+k)\sqrt{\frac{1-\gamma}{24}}\sqrt{\sigma_r}}{c_\rho \size \sqrt{\sigma_1}}\bigg)}{\log(1+\frac{1+\gamma}{2}\eta \sigma_r)}, \qquad T_2 = \frac{\log\left(\sqrt{\frac{24}{1-\gamma}}\right)}{\log(1+0.1\eta \sigma_r)},$$
	and 
	$$
	T_3= \frac{\log \left( \rho^{\frac{1-\gamma}{2(1+\gamma)}}/3\right)}{\log\left(1-\frac{3}{2}\eta \sigma_r \right)}, \qquad T= \frac{\log(\rho^{\frac{1-\gamma}{2(1+\gamma)}}/\rho)}{\log(1+\eta \gamma \sigma_r)} = \frac{\frac{3\gamma+1}{2(1+\gamma)}\log\left(\frac{1}{\rho}\right)}{\log(1+\eta \gamma \sigma_r)}.
	$$
	Using the bound 
	$$
	\rho \le \min\Bigg\{\frac{1}{3}, \frac{1-\gamma}{24}, \frac{c_\rho \sqrt{\sigma_1}}{12(m+n+k) \sqrt{\frac{1-\gamma}{24}}\sqrt{\sigma_r} } \Bigg\}^{\frac{180\gamma(1+\gamma)}{(3-2\gamma)(1-\gamma)}},
	$$
	we have 
	\begin{align}
	    \log \bigg(\frac{12(m+n+k)\sqrt{\frac{1-\gamma}{24}}\sqrt{\sigma_r}}{c_\rho \sqrt{\sigma_1}}\bigg) \le  \frac{(3-2\gamma)(1-\gamma)}{72\gamma}\log\left(\frac{1}{\size} \right) \label{eqn: forT1};
	\end{align}
	\begin{align}
	    \log\left( \sqrt{\frac{24}{1-\gamma}}\right) \le \frac{(3-2\gamma)(1-\gamma)}{360\gamma(1+\gamma)} \log \left(\frac{1}{\rho}\right); \label{eqn: forT2}
	\end{align}
	and 
	\begin{align}
	\log(3) \le \frac{(3-2\gamma)(1-\gamma)}{24\gamma(1+\gamma)} \log \left(\frac{1}{\rho}\right).
	    \label{eqn: forT3}
	\end{align}
	We will prove the following three inequalities: 
	\begin{enumerate}
	    \item $\frac{T_1}{T} \le \frac{4\gamma}{3\gamma+1} + \frac{(3-2\gamma)(1-\gamma)}{18(3\gamma+1)}$;
	    \item $\frac{T_2}{T} \le \frac{(3-2\gamma)(1-\gamma)}{18(3\gamma+1)}$;
	    \item $\frac{T_3}{T} \le \frac{2\gamma(1-\gamma)}{3(3\gamma +1)}+\frac{(3-2\gamma)(1-\gamma)}{18(3\gamma+1)}.$
	\end{enumerate}  
	The result will follow if we add these inequalities together. To prove the above inequalities, we make frequent use of the following two facts from Calculus: (i) For any $x \in (-1,1), \frac{x}{1+x} \le \log x \le x$. 
		(ii) For any $0<a<b$ and $x>0$, $\frac{\log(1+a\eta \sigma_r)}{\log(1+b\eta \sigma_r)}\le \frac{a}{b}$.
	
First, we have
\begin{align*}
	\frac{T_1}{T}&= \frac{\log (1+\gamma \eta \sigma_r)}{\log(1+\frac{1+\gamma}{2}\eta \sigma_r)} \frac{\log \bigg(\frac{12(m+n+k)\sqrt{\frac{1-\gamma}{24}}\sqrt{\sigma_r}}{c_\rho \size \sqrt{\sigma_1}}\bigg)}{\frac{3\gamma+1}{2(1+\gamma)}\log\left(\frac{1}{\rho}\right)} \\
	&\le \frac{2\gamma}{1+\gamma}\frac{\log \bigg(\frac{12(m+n+k)\sqrt{\frac{1-\gamma}{24}}\sqrt{\sigma_r}}{c_\rho \size \sqrt{\sigma_1}}\bigg)}{\frac{3\gamma+1}{2(1+\gamma)}\log\left(\frac{1}{\rho}\right)}\\
	&\le \frac{4\gamma}{3\gamma +1} \bigg( 1+ \frac{(3-2\gamma)(1-\gamma)}{72\gamma}\bigg)\\
	&= \frac{4\gamma}{3\gamma+1} + \frac{(3-2\gamma)(1-\gamma)}{18(3\gamma+1)},
\end{align*}
where the first inequality follows from fact 2, and the second inequality follows from~\eqref{eqn: forT1}.
Next, we have 
\begin{align*}
	\frac{T_2}{T} &= \frac{\log(1+\gamma \eta \sigma_r)}{\log(1+0.1\eta \sigma_r)}\frac{ \log\left(\sqrt{\frac{24}{1-\gamma}}\right)}{\frac{3\gamma+1}{2(1+\gamma)}\log\left(\frac{1}{\rho}\right)}\notag\\
	&\le \frac{20\gamma(1+\gamma)}{3\gamma +1} \frac{\log\left(\sqrt{\frac{24}{1-\gamma}}\right)}{\log\left(\frac{1}{\rho}\right)}\\
	&\le \frac{(3-2\gamma)(1-\gamma)}{18(3\gamma+1)},
\end{align*}
where the first inequality follows from fact 2 and the second inequality follows from~\eqref{eqn: forT2}.
Finally, we have
\begin{align*}
	\frac{T_3}{T} & = \frac{ \log(1+\eta \gamma \sigma_r)}{-\log\left(1-\frac{3}{2}\eta \sigma_r \right)} \frac{\log \left( 3/\rho^{\frac{1-\gamma}{2(1+\gamma)}}\right)}{\frac{3\gamma+1}{2(1+\gamma)}\log\left(\frac{1}{\rho}\right)}\notag\\
	& \le \frac{4\gamma(1+\gamma)}{3(1+3\gamma)} \frac{\log(3) + \frac{1-\gamma}{2(1+\gamma)} \log\left(\frac{1}{\rho}\right)}{ \log\left(\frac{1}{\rho}\right)}\\
	&\le \frac{2\gamma(1-\gamma)}{3(3\gamma +1)}+\frac{(3-2\gamma)(1-\gamma)}{18(3\gamma+1)},
\end{align*}
where the first inequality follows from fact 1 and the second inequality follows from~\eqref{eqn: forT3}.
Therefore, the result follows.
\end{proof}
\begin{lemma}[{\cite[Lemma 3.3]{ye2021global}}]
\label{lem: sigmarP}
	Suppose $P, \Sigma \in \RR^{r\times r}$ are two symmetric matrices, $\eta>0$, and $P' = (I-\eta(\Sigma - P))P(I-\eta(\Sigma - P))$. Suppose $\sigma_1(P)\le 2\sigma_1$ and $\sigma_r I \preceq \Sigma \preceq \sigma_1 I$. Then, for all $ \beta \in (0,1)$ and $\eta \le \frac{\beta}{8\sigma_1}$, it holds that 
	\begin{align*}
		\lambda_r(P') \ge \begin{cases}
			(1-\eta \sigma_r)^2 \lambda_r(P) - \frac{8+6\beta}{1-\beta} \eta^2 \sigma_1^3 & \text{if $\lambda_r(P)<0$}\\
			0 & \text{if $\lambda_r(P) \ge 0$}.
		\end{cases}
	\end{align*}
\end{lemma}

\subsection{Additional literature review}
\label{sec: more_lit_review}

The literature on gradient descent for matrix factorization is vast; see  \cite{chen2018harnessing,chi2019nonconvex} for a comprehensive survey of the literature, most of which focuses on the exact parametrization case $k=r$ (where $r$ is the target rank or the rank of certain ground truth matrix $X_\natural$) with the regularizer $\fnorm{F^\top F - G^\top G}^2$ that balances the magnitude of $F$ and $G$. Below we review recent progress on overparametrization for matrix factorization without additional regularizers. A summary of the results can be found in Table \ref{tab: comparison}. 

\begin{table}[t]
    \centering
    \begin{tabular}{|c|c|c|c|c|c|c|}
    \hline 
 & Asymmetric & Range of $k$&  Gap & $r$-SVD & Speed & Cold Start\\
 \hline 
\cite{li2018algorithmic}& \xmark & $k=m$ & 1 & \cmark & fast &\xmark \\
\cite{zhuo2021computational} &  \xmark & $k = \mathcal{O}(r)$ & $\geq 0.99$ & \xmark & n.a. &\xmark\\
 \cite{stoger2021small}& \xmark & $k\geq r$ & 1& \cmark & fast & \cmark\\
\cite{ye2021global} & \cmark & $k =r$ & 1& \cmark  & slow &\cmark\\
\cite{fan2020understanding} & \cmark & $k = 2m+2n$ & $\geq\frac{1}{2}$ & \xmark & n.a.  & \xmark\\
Our work  & \cmark & $k\geq r$ & $>0$ & \cmark & fast & \cmark\\\hline
    \end{tabular}
    \caption{A comparison of the settings and results from existing work on overparametrization matrix factorization. The column \emph{Asymmetric} summarizes whether the result applies to a general asymmetric matrix $X$. The column \emph{Range of $k$} shows the value of the rank parametrization $k$ to which the result applies. The column \emph{Gap} shows the requirement on the relative singular value gap $\delta = \frac{\sigma_r-\sigma_{r+1}}{\sigma_{r}}$; note that $\delta=1$ means $X$ is exactly rank-$r$. The column \emph{$r$-SVD} shows whether the analysis proves that gradient descent \eqref{eq: GD} converges arbitrarily close to $X_r$ (\cmark) or with an error bounded away from zero (\xmark). 
    In the column \emph{Speed}, we label the work as \emph{fast} if it shows that gradient descent converges to $X_r$ in a number of iteration that is logarithmic in the inverse of the final error $1/\epsilon$ and the dimension $m,n$ of $X$; we label it as \emph{slow} if the iteration complexity is polynomial in $1/\epsilon,m,n$; we write \emph{n.a.} if the error cannot be made arbitrarily small.  
    The last column \emph{Cold Start} shows whether the result allows the initial relative signal~\eqref{eq: initialiSignal} to be small (\cmark) or requires the signal to be larger than a universal constant (\xmark).} 
    \label{tab: comparison}
\end{table}

\paragraph{Matrix sensing with positive semidefinite matrices} 

A majority of existing theoretical work on overparametrization $(k>r)$ focuses on the matrix sensing problem: (approximately) recovering a positive semidefinite (psd) and rank-$r$ ground truth matrix $X_\natural$ from some linear measurement $b =\mathcal{A}(X_\natural)+e$, where $e$ is the noise vector and the linear constraint map $\mathcal{A}$ satisfies the so-called Restricted Isometry Property \cite[Definition 3.1]{recht2010guaranteed} (when $\mathcal{A}$ is the identity map, then this problem becomes the matrix factorization problem). These works analyze the gradient descent dynamics applied to the problem $\min \|\mathcal{A}(FF^\top)-b\|^2$.
In the noiseless ($e=0$) setting with an exactly rank-$r$ $X_\natural$, the pioneer work~\cite{li2018algorithmic} and subsequent improvement in~\cite{stoger2021small} show gradient descent recovers $X_\natural$ using random small initialization and arbitrary rank overparametrization $k$. In the noisy and approximately rank-$r$ setting, the work in~\cite{zhuo2021computational} shows that for arbitrary rank overparametrization, spectral initialization followed by gradient descent approximately recovers $X_\natural$ with a sublinear rate of convergence. However, their error bound for the gradient descent output with respect to $X_\natural$ scales with the overparametrization $k$, i.e., the algorithm overfits the noise under overparametrization. In particular, with $k=m$ ($m$ being the dimension of $X_\natural$), this error could be worse than that of the trivial estimator $0$. A similar limitation, that the error and/or sample complexity scales with $k$, also appears in earlier work on landscape analysis \cite{zhang2021sharp} as well as the recent work on preconditioned gradient descent \cite{zhang2021preconditioned} and subgradient methods \cite{ma2021sign,ding2021rank}. Existing results along this line all focus on positive semidefinite ground truth $X_\natural$ whose eigengap between the $r$-th and $(r+1)$-th eigenvalues is significant. In comparison, our results apply to a general asymmetric $X$ with arbitrary singular values, and our error bound depends only on the initialization size and stopping time but not $k$.

\paragraph{Matrix factorization and general asymmetric $X$} 

The work~\cite{ye2021global} also provides recovery guarantees for vanilla gradient descent \eqref{eq: GD} with random small initialization. Their result only applies to the setting where the matrix $X$ has exactly rank $r$ and $k=r$, i.e., with exact  parametrization. Moreover, their choice of stepsize is conservative and consequently their iteration complexity scales proportionally with the matrix dimension $m+n$. In comparison, we allow for significantly larger stepsizes and establish almost dimension-free iteration complexity bounds. 
To achieve an $\epsilon$ accuracy, the result in~\cite{ye2021global} requires $O\left(\frac{(m+n)^2 \sigma_1^4 r^4}{\sigma_r^4} \log\left(\frac{\sigma_r}{\epsilon}\right)\right)
$ iterations, while our main theorem only requires $O\left(\frac{\sigma_1^3}{\sigma_r^3}\log\left(\frac{\sigma_r}{\epsilon}\right)\right )$ iterations.
The work in~\cite{fan2020understanding} considers a wide range of statistical problems with a \emph{symmetric} ground truth matrix $X_\natural$, and shows that $X_\natural$ can be recovered with near optimal statistical errors using gradient descent for the objective function \eqref{eq: objective} with  $FG^\top$ replaced by  $FF^\top -GG^\top$.
While one may translate their results to the asymmetric setting via a dilation argument, it is feasible to do so  only under the specific rank parametrization $k =2m+2n$. This strong restriction on $k$ allows for a decoupling of the dynamics of different singular values, which is essential to their analysis. While this decoupled setting provides intuition for the general setting (as we elaborate in Section~\ref{section: earlyStoppingAndSmallIntialization}), the same analysis no longer applies for other values of $k$, e.g., $k=2m+2n-1$, in which case the singular values do not decouple. Moreover, a smaller value of $k$ leads to the cold start issue, as discussed in footnote~\ref{ft: sig}.

\paragraph{Deep matrix factorization} 

The work in~\cite{chou2020gradient} studies the deep matrix factorization problem of factorizing a given matrix $X$ into a product of multiple matrices. While on a high level their results deliver a message similar to our work---namely, gradient descent sequentially approaches the principal components of $X$---the technical details differ significantly. In particular, they results only apply to symmetric $X$ and guarantee recovery of the positive semidefinite part of $X$. Their analysis relies crucially on the assumption $k=m=n$, a specific identity initialization scheme and the resulting decoupled dynamics, which do not hold in the general setting as discussed above. A major contribution of our work lies in handling the entanglement of singular values resulted from general overparametrization, asymmetry, and random initialization.  

\subsection{The general symmetric setting}

In this section, we show that the arguments in Section~\ref{section: earlyStoppingAndSmallIntialization} can be generalized to the setting where the observed matrix $X$ is a general symmetric positive semidefinite (p.s.d.) matrix.

In particular, we illustrate the behavior of gradient descent with small random initialization in the setting with 
(i) $m=n=k$, and (ii) $X\in\mathbb{R}^{m\times m}$ is a p.s.d.\ matrix with eigen decomposition $X= U \Sigma U^\top$, where $\Sigma = \diag(\lambda_1,\lambda_2,\ldots,\lambda_m)$. Moreover, we assume that the $r$-th eigengap exists, i.e. $ \lambda_{r+1}\le \gamma\lambda_r$ for some $\gamma \in [0,1)$. The following argument works for any $\gamma <1$; for ease of presentation, we assume $\gamma = \frac{1}{10}$. We consider the natural objective function $f(F) = \frac{1}{4}\fnorm{FF^\top - X}^2$ and the associated gradient descent dynamic\footnote{One can obtain the same dynamic by using the initialization $F_0=G_0$ in \eqref{eq: GD}.}
$F_{t+1} = F_t - \eta(F_tF_t^\top -X)F_t$, with initialization $F_0 = \rho \sqrt{\lambda_1} I$ for some small $\rho>0$.
Let $\tilde F_t := U^\top F_t U$. To understand how $F_t F_t^\top$ approaches $X_r$, we can equivalently study how $\tilde F_t\tilde F_t^\top$ approaches $\Sigma_r$. By simple algebra, we have
\begin{align}\label{eqn: fanexampleupdatesupp}
\tilde F_{t+1} = \tilde F_t - \eta (\tilde F_t \tilde F_t^\top - \Sigma) \tilde F_t.
\end{align}
Since $\tilde F_0 = F_0$ is diagonal, it is easy to see see that $\tilde F_t$ is  \emph{diagonal} for all $t \ge 0$, and the $i$-th diagonal element in $\tilde F_t$, denoted by $f_{i,t}$,  is updated as
\begin{equation}\label{eq: eigenvalueDynamicsupp}
f_{i,t+1} = f_{i,t}(1+\eta \lambda_i - \eta f_{i,t}^2).
\end{equation}
Thus, the dynamics of all the eigenvalues decouples and can be analyzed separately. In particularly,  simple algebra shows that for $1\le i \le r$, (i)  when $f_{i,t} < \sqrt{\frac{\lambda_i}{2}}$ ,  $f_{i,t}$  increases geometrically by a factor of $1+\eta \lambda_i - \eta f_{i,t}^2 \ge 1+ \frac{\eta \lambda_r}{2} $, i.e., $f_{i,t+1} \ge \left( 1+ \frac{\eta \lambda_r}{2}\right) f_{i,t}$, and (ii) when $f_{i,t}\ge \sqrt{\frac{\lambda_i}{2}}$ holds, $f_{i,t}$ converges to $\sqrt{\lambda_i}$ geometrically because
\begin{align*}
|f_{i,t+1}-\sqrt{\lambda_i}| &= |f_{i,t}-\sqrt{\lambda_i}||1-\eta f_{i,t}( f_{i,t}+\sqrt{\lambda_i})|\\
&\le (1-\frac{\eta \lambda_i}{2})|f_{i,t}-\sqrt{\lambda_i}|
\le (1-\frac{\eta \lambda_r}{2})|f_{i,t}-\sqrt{\lambda_i}|.
\end{align*}
In summary, the first $r$ diagonal elements of $\tilde F_t$ will first increase geometrically by a factor of (at least) $1+\frac{\eta \lambda_r}{2}$, and then converge to $\sqrt{\lambda_i}$ geometrically by a factor of  $1-\frac{\eta\lambda_r}{2}$.

What makes a difference, however, is that for $ i\ge r+1$,  $f_{i,t}$ converges at an exponentially slower rate than the first $r$ diagonal elements. In particular, assuming the step size $\eta$ is sufficiently small, for $i \ge r+1 $, we can show that $f_{i,t}$ is always nonnegative and satisfies
\begin{align}
f_{i,t+1} &= f_{i,t}(1+\eta \lambda_i - \eta f_{i,t}^2) \label{eqn: updateofeigenintro}\\
&\le (1+\eta \lambda_i)f_{i,t}\notag
\le (1+\frac{\eta \lambda_r}{10})f_{i,t} \notag.
\end{align} 
Note that the growth factor $1+\frac{\eta \lambda_r}{10}$ is smaller than $1+\frac{\eta \lambda_r}{2}$, the growth factor for $f_{i,t}, i\le r.$ Consequently, we conclude that 
\textit{larger singular value converges (exponentially) faster.}
This property implies that $F_tF_t^\top$ approaches $X_1,X_2,X_3\ldots$ sequentially for a positive semidefinite $X$ with distinct singular values, since we can repeat the above argument for each $r = 1,2,3,\ldots$

\end{document}

%% file: liweimacros.tex
\numberwithin{equation}{section}
\usepackage{enumitem}

\newcommand{\norm}[1] {\left \| #1 \right \|}
\newcommand{\inclu}[0] {\ar@{^{(}->}}

\newcommand{\diag}{{\rm diag}}

\newcommand{\RR}{\mathbb{R}}

\newcommand{\rank}{\mathrm{rank}}

\newcommand{\fnorm}[1]{\| #1 \|_{ {\tiny \text{F}}}}

\newcommand{\floor}[1]{\left \lfloor #1 \right \rfloor }

\newcommand{\argmin}{\operatornamewithlimits{argmin}}

\newcommand{\dotp}[1]{\left\langle #1\right\rangle}

\newtheorem{theorem}{Theorem}[section]
\newtheorem{lemma}{Lemma}[section]

\newtheorem{proposition}[theorem]{Proposition}
\newtheorem{corollary}[theorem]{Corollary}

\newtheorem{assumption}{Assumption}

\theoremstyle{remark}

\usepackage{mathtools}
\usepackage[boxruled]{algorithm2e}